\documentclass[nohyperref]{article}

\usepackage{microtype}
\usepackage{graphicx}
\usepackage{subcaption}
\usepackage{booktabs} % for professional tables

\usepackage{hyperref}
\hypersetup{colorlinks}
\usepackage{url}            % simple URL typesetting

\usepackage[accepted]{icml2023}

\usepackage[utf8]{inputenc} % allow utf-8 input
\usepackage[T1]{fontenc}    % use 8-bit T1 fonts
\usepackage{hyperref}       % hyperlinks
\usepackage{booktabs}       % professional-quality tables
\usepackage{amsfonts}       % blackboard math symbols
\usepackage{nicefrac}       % compact symbols for 1/2, etc.
\usepackage{microtype}      % microtypography
\usepackage{xcolor}         % colors
\usepackage[title]{appendix}
\usepackage{todonotes}

\usepackage{graphicx,verbatim,lineno,titletoc}
\usepackage{amssymb}
\usepackage{amsmath, amsthm}
\usepackage[all]{xy}
\usepackage{enumerate}
\usepackage[english]{babel}
\usepackage{multirow}
\usepackage{float}
\usepackage{enumitem}
\usepackage{accents}
\usepackage[capitalize,noabbrev]{cleveref}

\theoremstyle{plain}
\newtheorem{theorem}{Theorem}[section]
\newtheorem{prop}[theorem]{Proposition}
\newtheorem{lemma}[theorem]{Lemma}

\theoremstyle{definition}
\newtheorem{definition}[theorem]{Definition}
\newtheorem{assumption}[theorem]{Assumption}
\theoremstyle{remark}

\usepackage{pifont}

\def\liminf{\mathop{\rm lim\,inf}\limits}

\def\R{\mathbb{R}}

\def\eps{\varepsilon}

\def\D{\mathbf{D}}
\def\X{\mathbf{X}}

\def\A{\mathbf{A}}
\def\B{\mathbf{B}}

\def\param{\boldsymbol{\theta}}
\def\Param{\boldsymbol{\Theta}}

\def\M{\mathbf{M}}

\DeclareMathOperator{\Out}{\texttt{Out}}

\DeclareMathOperator*{\argmin}{arg\,min}
\DeclareMathOperator{\vect}{vec}

\def\liminf{\mathop{\rm lim\,inf}\limits}

\def\R{\mathbb{R}}

\def\eps{\varepsilon}

\def\X{\mathbf{X}}

\def\A{\mathbf{A}}
\def\B{\mathbf{B}}

\icmltitlerunning{Complexity of BCD-PR and Applications to Wasserstein CP-Dictionary Learning}

\begin{document}

\twocolumn[
\icmltitle{Complexity of Block Coordinate Descent with Proximal Regularization \\ and Applications to Wasserstein CP-dictionary Learning}

% It is OKAY to include author information, even for blind
% submissions: the style file will automatically remove it for you
% unless you've provided the [accepted] option to the icml2023
% package.

% List of affiliations: The first argument should be a (short)
% identifier you will use later to specify author affiliations
% Academic affiliations should list Department, University, City, Region, Country
% Industry affiliations should list Company, City, Region, Country

% You can specify symbols, otherwise they are numbered in order.
% Ideally, you should not use this facility. Affiliations will be numbered
% in order of appearance and this is the preferred way.
\icmlsetsymbol{equal}{*}

\begin{icmlauthorlist}
\icmlauthor{Dohyun Kwon}{equal,yyy}
\icmlauthor{Hanbaek Lyu}{equal,xxx}
\end{icmlauthorlist}

\icmlaffiliation{yyy}{Department of Mathematics, University of Seoul, Seoul, Republic of Korea}
\icmlaffiliation{xxx}{Department of Mathematics, University of Wisconsin - Madison, Wisconsin, United States}

\icmlcorrespondingauthor{Dohyun Kwon}{dh.dohyun.kwon@gmail.com}
\icmlcorrespondingauthor{Hanbaek Lyu}{hlyu@math.wisc.edu}

% You may provide any keywords that you
% find helpful for describing your paper; these are used to populate
% the "keywords" metadata in the PDF but will not be shown in the document
\icmlkeywords{Dictionary learning, CONDECOMP/PARAFAC, Block coordinate descent, Optimal transport, Nonconvex optimization}

\vskip 0.3in
]

% this must go after the closing bracket ] following \twocolumn[ ...

% This command actually creates the footnote in the first column
% listing the affiliations and the copyright notice.
% The command takes one argument, which is text to display at the start of the footnote.
% The \icmlEqualContribution command is standard text for equal contribution.
% Remove it (just {}) if you do not need this facility.

\printAffiliationsAndNotice{}  % leave blank if no need to mention equal contribution
%\printAffiliationsAndNotice{\icmlEqualContribution} % otherwise use the standard text.

\begin{abstract}
We consider the block coordinate descent methods of Gauss-Seidel type with proximal regularization (BCD-PR), which is a classical  method of minimizing general nonconvex objectives under constraints that has a wide range of practical applications. We theoretically establish the worst-case complexity bound for this algorithm. Namely, we show that for general nonconvex smooth objective with block-wise constraints, the classical BCD-PR algorithm converges to an $\eps$-stationary point within $\widetilde{O}(\eps^{-1})$ iterations. Under a mild condition, this result still holds even if the algorithm is executed inexactly in each step. As an application, we propose a provable and efficient algorithm for `Wasserstein CP-dictionary learning',  which seeks a set of elementary probability distributions that can well-approximate a given set of $d$-dimensional joint probability distributions. Our algorithm is a version of BCD-PR that operates in the dual space, where the primal problem is regularized both entropically and proximally. 
\end{abstract}

\section{Introduction}
Consider the minimization of a continuous function $f:\R^{I_{1}}\times \dots \times R^{I_{m}}\rightarrow [0,\infty)$ on a cartesian product of convex sets $\Param=\Theta^{(1)}\times \dots \times \Theta^{(m)}$:
	\begin{align}\label{eq:block_minimization}
		\param^{*} \in \argmin_{\param=[\theta_{1},\dots,\theta_{m}]\in \Param} f(\theta_{1},\dots,\theta_{m}) .
	\end{align} 
When the objective function $f$ is nonconvex, the convergence of any algorithm for solving \eqref{eq:block_minimization} to a globally optimal solution can hardly be expected. Instead, global convergence to stationary points of the objective function is desired, and in some problem classes, stationary points could be as good as global optimizers either practically as well as theoretically (see \cite{mairal2010online, sun2015nonconvex}).
	
In order to solve \eqref{eq:block_minimization}, we will consider the \textit{block coordinate descent} (BCD) methods of Gauss–Seidel type, which seeks to minimize the objective function restricted to a subset (block) of coordinates \cite{wright2015coordinate}, often following the cyclic order of blocks. For the minimization problem \eqref{eq:block_minimization} we refer to the set of coordinates in each $\Theta^{(i)}$, $i=1,\dots,m$, a block coordinate. Namely, let $\theta^{(i)}_{n}$ denote the $i$th block of the parameter after $n$ updates. Write
\begin{align}\label{eq:def_marginal_ft}
\param_{n}^{(i-1)}&:=(\theta_{n}^{(1)},\cdots,\theta_{n}^{(i)},\theta_{n-1}^{(i+1)},\cdots, \theta_{n-1}^{(m)}), \\
 f_{n}^{(i)}(\theta)&:=f\left(\theta_{n}^{(1)},\cdots,\theta_{n}^{(i-1)},\theta,\theta_{n-1}^{(i+1)},\cdots, \theta_{n-1}^{(m)}\right). \nonumber
\end{align} 
The algorithm we consider in this work updates $\param_{n}^{(i-1)}$ to $\param_{n}^{(i)}$ by updating its $i$th block by minimizing the marginal loss function $g_{n}^{(i)}$ over the $i$th block $\Theta^{(i)}$:
\begin{align}\label{eq:BCD_factor_update_proximal}
\hspace{-0.2cm}	\theta_{n}^{(i)} &\leftarrow \argmin_{\theta\in \Theta^{(i)}} g_{n}^{(i)}(\theta):=f_{n}^{(i)}(\theta)  + \frac{\lambda_{n}}{2} \lVert \theta - \theta_{n-1}^{(i)} \rVert^{2},
\end{align}
where $\lambda_{n}\ge 0$ is called proximal regularization coefficient and $\lVert \cdot \rVert$ denotes the Frobenius norm. The proximal regularzer $\lambda_{n} \lVert \theta - \theta_{n-1}^{(i)} \rVert^{2}$ ensures that the next block iterate $\theta^{(i)}_{n}$ is not too far from the previous iterate  $\theta_{n-1}^{(i)}$. The above update is applied cyclicly for $i=1,\dots,m$. We call the algorithm \eqref{eq:BCD_factor_update_proximal} BCD-PR for block coordinate descent with proximal regularization.  

Due to its simplicity, BCD type algorithms have been applied to a wide range of nonconvex problems \cite{bottou2010large}, including matrix and tensor decomposition problems such as nonnegative matrix factorization \cite{lee1999learning, lee2001algorithms, wang2012nonnegative} and nonnegative CANDECOMP/PARAFAC (CP) decomposition \cite{tucker1966some, harshman1970foundations, carroll1970analysis}. Notably, all these decomposition problems enjoy block multi-convex structure, wherein the objective function is convex when restricted on each block coordinate so that each convex sub-problems can be solved via standard convex optimization algorithms \cite{boyd2004convex}. However, such multi-convexity is not required to apply BCD, as simple coordinate-wise gradient descent can be applied to find the approximate minimizer of the sub-problems \cite{wright2015coordinate}. 

It is known that vanilla BCD (\eqref{eq:BCD_factor_update_proximal} with $\lambda_{n}\equiv 0$) does not always converge to the stationary points of the non-convex objective function that is convex in each block coordinate \cite{powell1973search, grippo2000convergence}. It is known that BCD-PR with $\lambda_{n}\equiv Const.$ is guaranteed to converge to the set of stationary points \cite{grippo2000convergence}. Under a more general condition, BCD-PR and its prox-linear variant are shown to converge to Nash equilibria. Local convergence result with rate is known for these algorithms under the stronger condition of Kurdyka-\L{}ojasiewicz \cite{attouch2010proximal, xu2013block, bolte2014proximal}. For convex objectives, iteration complexity of $O(\eps^{-1})$ is established in \citet{hong2017iteration}. The BCD method has been drawing attention as an alternative method for training Deep Neural Network (DNN) models. In \citet{zhang2017convergent}, a BCD method is shown to converge to stationary points for Tikhonov regularized DNN models. In \citet{zeng2019global}, BCD-PR for training DNNs with general activation functions is shown to have iteration complexity of $O(\eps^{-1})$.

\textbf{Contribution.} 
While being one of the fundamental nonconvex optimization methods, the worst-case iteration complexity of BCD-PR \eqref{eq:BCD_factor_update_proximal} for general objectives under constraints has not been established in the literature. We intend to fill this gap with contributions summarized below: 
	\begin{enumerate}[label={$\bullet$}]
		\item  Global convergence to stationary points of BCD-PR for $L$-smooth objective $f$ under constraints; 
		\item Worst-case bound  of $O(\eps^{-1} (\log \eps^{-1})^{2})$ on the number of iterations to achieve $\eps$-approximate stationary points; 
		\item Robustness of the aforementioned results under inexact execution of the algorithm.
	\end{enumerate}
To our best knowledge, we believe our work provides the first result on the global rate of convergence and worst-case iteration complexity of BCD-PR for the general smooth objectives, especially with the additional robustness result. For gradient descent methods with unconstrained nonconvex objective, it is known that such rate of convergence cannot be faster than $O(\eps^{-1})$ \cite{cartis2010complexity}, so our rate bound matches the optimal result up to a $(\log \eps^{-1})^{2}$ factor. We emphasize that the above result does not claim that BCD-PR is provably faster than existing non-convex optimization algorithms. Instead, our novel analysis confirms that the classic and practical algorithm of BCD-PR is guaranteed to converge as fast as existing algorithms in the worst case. 
   
The works \cite{attouch2010proximal} and \cite{bolte2014proximal} assume that the objective function satisfies KL property at every point in the parameter space and obtains a global rate of convergence to a stationary point for block proximal Gauss-Seidel (equivalent to our Algorithm 1) and block proximal alternating linearized minimization. On the other hand, \citet{xu2013block} assumed local KL property and obtained a local rate of convergence to a stationary point for both types of BCD methods. In our work, we do not assume KL property at any point and still obtain a global convergence rate for block proximal Gauss-Seidel.

\textbf{Application to Wasserstein CP-dictionary learning}.  In order to motivate our theoretical underpinning of BCD-PR, we consider the problem of Wasserstein CP-dictionary learning for $d$-dimensional joint distributions, which seeks a set of elementary probability distributions that can well-approximate a given set of $d$-dimensional joint probability distributions represented as $d$-mode tensors. 
\begin{enumerate}[label={$\bullet$}]
\item  We propose the \textit{Wasserstein CP-dictionary learning} (WCPDL) framework for learning elementary probability distributions that reconstruct $d$-dimensional joint probability distributions represented as $d$-mode tensors. 
\item We propose an algorithm for WCPDL based on BCD-PR, where the sub-problems of Wasserstein reconstruction error minimization are handled by using entropic regularization and dual formulation for computational efficiency.
\item We establish worst-case bound  of $O(\eps^{-1} (\log \eps^{-1})^{2})$ on the number of iterations to achieve $\eps$-approximate stationary points for WCPDL.
\end{enumerate}
We also demonstrate the advantage of the Wasserstein formulation for distribution-valued dictionary learning through a number of experiments and applications.

\section{Preliminaries}
Before stating our main results in the following sections, let us recall a list of definitions for \eqref{eq:block_minimization}. We say $\param^{*}\in \Param$ is a \textit{stationary point} of a function $f$ over $\Param$ if 
\begin{align}\label{eq:stationary}
    \inf_{\param\in \Param} \, \langle \nabla f(\param^{*}) ,\,  \param - \param^{*} \rangle \ge 0,
\end{align}	
where $\langle \cdot,\, \cdot \rangle$ denotes the dot project on $\R^{I_{1}+\dots+I_{m}}\supseteq \Param$. This is equivalent to saying that $-\nabla f(\param^{*})$ is in the normal cone of $\Param$ at $\param^{*}$. If $\param^{*}$ is in the interior of $\Param$, then it implies $\lVert \nabla f(\param^{*}) \rVert=0$. For iterative algorithms, such a first-order optimality condition may hardly be satisfied exactly in a finite number of iterations, so it is more important to know how the worst-case number of iterations required to achieve an $\eps$-approximate solution scales with the desired precision $\eps$. More precisely, we say $\param^{*}\in \Param$ is an \textit{$\eps$-approxiate stationary point} of $f$ over $\Param$ if 
\begin{align}\label{eq:stationary_approximate}
    -\inf_{\param\in \Param} \, \left\langle \nabla f(\param^{*}),\,  \frac{(\param - \param^{*}) }{\lVert \param-\param^{*} \rVert} \right\rangle\le \sqrt{\eps}.
\end{align}	
This notion of $\eps$-approximate solution is consistent with the corresponding notion for unconstrained problems. Indeed, if $\param^{*}$ is an interior point of $\Param$, then \eqref{eq:stationary_approximate} reduces to $\lVert \nabla f(\param^{*}) \rVert^{2}\le \eps$.  It is also equivalent to a similar notion in Def. 1 in \citet{nesterov2013gradient}, which is stated for non-smooth objectives using subdifferentials instead of gradients as in \eqref{eq:stationary_approximate}. Next, for each $\eps>0$ we define the \textit{worst-case iteration complexity} $N_{\eps}$ of an algorithm computing $(\param_{n})_{n\ge 1}$ for solving \eqref{eq:block_minimization} as 
\begin{align}\label{eq:Neps}
    N_{\eps}:= \sup_{\param_{0}\in \Param} \, \inf\, \left\{ n \,|\, \begin{matrix} \text{$\param_{n}$ is an $\eps$-approximate } \\  \textup{stationary point of $f$ over $\Param$} \end{matrix}\right\}, 
\end{align}
where $(\param_{n})_{n\ge 0}$ is a sequence of estimates produced by the algorithm with an initial estimate $\param_{0}$. Note that $N_{\eps}$ gives the \textit{worst-case} bound on the number of iterations for an algorithm to achieve an $\eps$-approximate solution due to the supremum over the initialization $\param_{0}$ in \eqref{eq:Neps}.
	
\section{Statement of the results}

We state the main result, Theorem \ref{thm:BCD}. To our best knowledge, this gives the first worst-case rate of convergence and iteration complexity of BCD-type algorithms with proximal regularization in the literature. %We remark that a similar result for BCD with diminishing radius was recently obtained in \cite{lyu2020convergence}. 
We impose the following two mild conditions for our theoretical analysis of BCD-PR \eqref{eq:BCD_factor_update_proximal}. 
\begin{assumption}\label{assumption:A1}
For each $i = 1,2, \cdots, m$, there exists a constant $L^{(i)}>0$ such that the function $f: \boldsymbol{\Theta}= \Theta^{(1)}\times \dots \times \Theta^{(m)}\rightarrow [0,\infty)$ is $L^{(i)}$-smooth in each block coordinate $i$, that is, the function $\theta \mapsto \nabla f(\theta^{(1)},\cdots,\theta^{(i-1)},\theta,\theta^{(i+1)},\cdots, \theta^{(m)})$ is $L^{(i)}$-Lipschitz in $\Theta^{(i)}$ for any $\theta^{(j)} \in \Theta^{(j)}, j=1,2, \cdots, i-1, i+1, \cdots, m$.		
\end{assumption}
 
\begin{assumption} \label{assumption:A2}
The constraint sets  $\Theta^{(i)}\subseteq \R^{I_{i}}$, $i=1,\dots,m$ are convex. Furthermore, the sub-level sets $f^{-1}((-\infty, a))=\{\param\in \Param\,:\, f(\param)\le a \}$ are compact for each $a\in \R$. 
\end{assumption}
	
We also allow an inexact computation of the solution to the sub-problem \eqref{eq:BCD_factor_update_proximal}. For a quantitative statement, for each $n\ge 1$, we define the \textit{optimality gap} $\Delta_{n}$ by
\begin{align}\label{eq:def_sub_optimality_gap}
    \Delta_{n}:=	\max_{1\le i \le m}  \left( g_{n}^{(i)} (\theta^{(i)}_{n}) - \inf_{\theta\in \Theta^{(i)}} g_{n}^{(i)} (\theta) \right), 
\end{align}
where $g_{n}^{(i)}$ is in \eqref{eq:BCD_factor_update_proximal}. For our convergence results to hold, we require the optimality gaps to decay sufficiently fast so that they are summable:
\begin{assumption}\label{assumption:A3}
The optimality gaps $\Delta_{n}$ are summable, that is, $\sum_{n=1}^{\infty} \Delta_{n} <\infty$. 
\end{assumption}

We now state our main result for BCD-PR. 
	
\begin{theorem}\label{thm:BCD}
Let $(\param_{n})_{n\ge 0}$ be an inexct output of \eqref{eq:BCD_factor_update_proximal}.
Suppose that Assumptions~\ref{assumption:A1}-\ref{assumption:A3} hold. 
Let $L^{(i)}>0$ be such that $\nabla f$ is $L^{(i)}$-Lipschitz in each block coordinate and suppose the proximal regularizers $(\tau_{n}^{(i)})_{n\ge 1}$ satisfy $\tau_{n}^{(i)}> L^{(i)}$ for $n\ge 1$ and $\tau_{n}=O(1)$. Then the following hold: 
\begin{description}
\item[(i)] (Global convergence to stationary points) Every limit point of $(\param_{n})_{n\ge 0}$ is a stationary point of $f$ over $\Param$.

\vspace{0.1cm}
\item[(ii)]  (Worst-case rate of convergence)  There exists a constant $M$ independent of $\param_{0}$ such that for  $n\ge 1$,
\begin{align} \nonumber
    \min_{1\le k \le n}  	 \,\, \left[ -\inf_{\param\in \Param} \left\langle \nabla f(\param_{k}) ,\, \frac{(\param - \param_{k}) }{\lVert \param - \param_{k}\rVert}\right\rangle \right]^{2} \\ \le  \frac{M + 2m\sum_{n=1}^{\infty}\Delta_{n} }{ n/(\log n)^{2} }.\label{eq:thm_convergence_bd_inexact}
\end{align}

\vspace{0.1cm}
\item[(iii)] (Worst-case iteration complexity) Suppose the optimality gaps are uniformly summable, that is, $\sup_{\param_{0}\in \Param} \sum_{n=1}^{\infty} \Delta_{n} <\infty$. Then the worst-case iteration complexity $N_{\eps}$ for  BCD-PR \eqref{eq:BCD_factor_update_proximal} satisfies $N_{\eps} = O(\eps^{-1} (\log \eps^{-1})^{2} )$ if $\tau_{n}\equiv 1$.
\end{description}
\end{theorem}    

\section{Application to $d$-dimensional Wasserstein dictionary learning}

We apply our optimization method of BCD-PR \eqref{eq:BCD_factor_update_proximal} to solve \textit{$d$-dimensional Wasserstein dictionary learning}, where the goal is to learn a dictionary of product probability distributions from a set of joint distributions. Namely, given $d$-dimensional joint probability distributions  $(\X_{k})_{1\le k \le N}$, we seek to find a set of product distributions such that each $\X_{k}$ can be approximated by a suitable mixture of the product distributions. 

\subsection{Dictionary learning for distribution-valued signals}

For $N$ observed $d$-mode tensor-valued signals $\mathbf{X}_{1},\dots,\mathbf{X}_{N}$ in $\R^{I_{1}\times \dots \times I_{d}}$, we are interested in extracting $r$ `features' from this set, where each feature again takes the form of $d$-mode tensors in $\R^{I_{1}\times \dots \times I_{d}}$. In other words, we seek to learn a `dictionary' $\mathcal{D}=[\D_{1},\dots,\D_{r}]\in \R^{I_{1}\times \dots \times I_{d} \times r}$ of $r$ `atoms' so that each data tensor $\mathbf{X}_{i}$ can be linearly approximated by the atoms $\D_{1},\dots,\D_{r}$ in the dictionary $\mathcal{D}$. Namely, there exists a suitable `code matrix' $\Lambda\in \R^{r\times N}$ such that we have the following approximate factorization:
\begin{align}\label{eq:tensor_DL_1}
	[\mathbf{X}_{1},\dots, \mathbf{X}_{N}] &\approx [\mathbf{D}_{1},\dots, \mathbf{D}_{r}]\times_{d+1} \Lambda\\ \qquad \Longleftrightarrow \qquad \mathcal{X} &\approx \mathcal{D}\times_{d+1} \Lambda, \nonumber
\end{align}
where $\times_{d+1}$ denotes the mode $(d+1)$ tensor-matrix product (see \cite{kolda2009tensor}) and $\mathcal{X}:=[\X_{1},\dots,\X_{n}]$ denotes the $(d+1)$-mode tensor in $\R^{I_{1}\times \dots \times I_{d} \times N}$ that concatenates the tensor-valued signals $\X_{1},\dots,\X_{n}$ in $\R^{I_{1}\times \dots \times I_{d}}$ along the last mode. As a special case, suppose $d=1$ so that the signals $\X_{1},\dots,\X_{n}$ are in fact $I_{1}$-dimensional vectors. Then \eqref{eq:tensor_DL_1} becomes the usual matrix factorization formulation for factorizing the data matrix $\mathcal{X}\in \R^{I_{1}\times N}$ into the (matrix) product of a dictionary matrix $\mathcal{D}\in \R^{I_{1}\times r}$ and the code matrix $\Lambda \in \R^{r\times N}$ \cite{lee1999learning, elad2006image, mairal2007sparse, peyre2009sparse}. 

As a more precise optimization formulation of \eqref{eq:tensor_DL_1}, we consider 
\begin{align}\label{eq:tensor_DL_2}
    \min_{\mathcal{D} \in \R^{I_{1}\times \cdots \times I_{d} \times r}, \Lambda \in \R^{r\times N}} \delta\bigg( \mathcal{X},\,  \mathcal{D} \times_{d+1} \Lambda  \bigg),
\end{align}
where $\delta: (\R^{I_{1}\times \dots \times I_{d}\times N})^{2}\rightarrow [0,\infty)$ is a `dissimilarity function' that maps a pair of tensors $(\mathcal{X},\mathcal{X}')$ to a nonnegative number $\delta(\mathcal{X}, \mathcal{X}')$. This function is used to measure the difference between the data tensor $\mathcal{X}$ and the `reconstruction' $\mathcal{D} \times_{d+1} \Lambda$. For $d=1$, standard choices of $\delta$ include the distance function induced by the Frobenius norm and the KL divergence. 

\subsection{Wasserstein distance between $d$-dimensional probability distributions} 
\label{ap:w}

A natural notion of dissimilarity between two probability distributions on the same probability space is the \textit{$p$-Wasserstein distance}, which is a central notion in this paper, which we will define below. 

Define the \textit{cost tensor} $\M\in \R^{I_{1}\times \cdots \times I_{d}}\times \R^{I_{1}\times \cdots \times I_{d}}$ for $d$-mode tensors to be the tensor defined by  $\M(J_{1},J_{2})=\lVert J_{1}-J_{2}\rVert_{2}$ for all multi-indices $J_{1},J_{2}\in [I_{1}]\times \dots \times [I_{d}]$. 

One can regard $\M$ as giving weights on the difference between the $J_{1}$- and the $J_{2}$-entry of two tensors. For instance, if $d=1$, then the dissimilarity between the two random variables $Y_{1}$ and $Y_{2}$ depends not only on the probability that they differ but also on the actual value $|Y_{1}-Y_{2}|$. The cost matrix $\M$, in this case, measures the probabilistic `cost' of having different probability mass on coordinates $J_{1}$ and $J_{2}$. Next, for two one-dimensional probability mass functions $p_{1}\in \R^{m}$, $p_{2}\in \R^{n}$, we call a two-dimensional joint distribution $T\in \Sigma_{m,n}$ a \textit{coupling} between $p_{1}$ and $p_{2}$ if its row (resp., column) sums agree with $p_{1}$ (resp., $p_{2}$). We denote by 
\begin{align*}
    U(p_{1},p_{2})&:= \left\{ T\in \Sigma_{m,n}\,\bigg|\, p_{1}(i)=\sum_{j=1}^{n} T(i,j), p_{2}(j)= \right. \\   &\left. \sum_{i=1}^{m} T(i,j)  \, \forall i\in \{1,\dots,m\},\, j\in \{1,\dots,n\}  \right\}
\end{align*}
the set of all couplings between $p_{1}$ and $p_{2}$. 

Now, we can define the Wasserstein distance. Fix a cost tensor $\M\in \R^{I_{1}\times \cdots \times I_{d}}\times \R^{I_{1}\times \cdots \times I_{d}}$ and let $\M^{2}\in \R^{(I_{1} \cdots  I_{d}) \times (I_{1} \cdots  I_{d})}$ denote its matricization (see \cite{kolda2009tensor}). Fix a parameter $\gamma\ge 0$. For $\A,\B\in \R^{I_{1}\times \cdots \times I_{d}}$, define 
    \begin{align}
        W_{\gamma}(\A,\B) &:= W_{\gamma}(\vect(\A),\vect(\B)) \nonumber\\
        &:= \min_{T\in U(\vect(\A), \vect(\B))}  \left\langle \M^{2}, T \right\rangle + \gamma \left\langle T, \log T \right\rangle, \label{eq:Wasserstein_def}
    \end{align}
where $\vect(\A)$ and $\vect(\B)$ denote the vectorization of $\A$ and $\B$, respectively. When $\gamma=0$, $W_{\gamma}$ above is known as the Wasserstein distance. The additional term $\gamma \langle T, \log T \rangle$ is known as the \textit{entropic regularization} of Wasserstein distance \cite{cuturi2013sinkhorn}.

\subsection{$d$-dimensional Wasserstein dictionary learning} 

We are interested in the case that the tensor-valued signals $\X_{1},\dots,\X_{N}$ describe $d$-dimensional probability mass functions. Namely, we denote 
\begin{align*}
    \Sigma_{I_{1},\dots,I_{d}}:=\left\{ \X\in \R_{\ge 0}^{I_{1}\times \dots \times I_{d}}\,\bigg|\, \sum_{i_{1},\dots,i_{d}} \X[i_{1},\dots,i_{d}]=1     \right\}.
\end{align*}
We can think of an element $\X$ of $\Sigma_{I_{1},\dots,I_{d}}$ as the joint probability mass function of $d$ discrete random variables $(Y_{1},\dots,Y_{d})$ where each $Y_{i}$ takes values from $\{1,\dots, I_{i} \}$. For this reason, we will call an element of $\Sigma_{I_{1},\dots,I_{d}}$ simply as a `$d$-dimensional joint distribution'. We also denote by $\Sigma_{I_{1},\dots,I_{d}}^{N}$ the $N$-fold product of $\Sigma_{I_{1},\dots,I_{d}}$, which we identify as a subset of $\R^{I_{1}\times \dots \times I_{d}\times N}$ in the usual way. 

When each $d$-mode tensor $\X_{i}$ subject to the factorization in \eqref{eq:tensor_DL_2} is a  $d$-dimensional joint distribution, then the dissimilarity function $\delta$ in \eqref{eq:tensor_DL_2} should measure the dissimilarity between two tuples of $d$-dimensional joint distribution. By using the \textit{entropy-regularized Wasserstein distance} $W_{\gamma}$ (see \eqref{eq:Wasserstein_def}), we formulate the \textit{$d$-dimensional Wasserstein Dictionary Learning} (dWDL) as \eqref{eq:tensor_DL_1}, where the dictionary atoms $\D_{1},\dots,\D_{r}$ are taken to be $d$-dimensional joint distributions (elements of $\Sigma_{I_{1},\dots,I_{d}}$) and  the dissimilarity function $\delta:\Sigma_{I_{1},\dots,I_{d}}^{N}\times \Sigma_{I_{1},\dots,I_{d}}^{N}\rightarrow [0,\infty)$ is
\begin{align}
    \delta([\X_{1},\dots,\X_{N}],\, [\X_{1}',\dots,\X_{N}']):=\sum_{i=1}^{N} W_{\gamma}(\X_{i}, \X_{i}').  \nonumber
\end{align}
Equivalently, we formulate our problem (dWDL) as below:
\begin{align}\label{eq:dWDL_main_problem}    
 \textup{\textbf{(dWDL)}} \hspace{0.5cm}
 \min_{\substack{\mathcal{D}=[\D_{1},\dots,\D_{r}]\in \Sigma_{I_{1},\dots,I_{d}}^{r} \\ \Lambda\in \Sigma_{r}^{N} }} f_W(\mathcal{D}, \Lambda),\\
 \hbox{ where } f_W(\mathcal{D}, \Lambda):= \sum_{i=1}^{N}W_{\gamma}\left(\X_i, \, \mathcal{D}\times_{d+1} \Lambda[:,i] \right). \nonumber
\end{align}
\vspace{-0.1cm}
For $d=1$, this formulation \eqref{eq:dWDL_main_problem} has been discussed in the study of Wasserstein dictionary learning, including \cite{sandler2011nonnegative}, \cite{zen2014simultaneous}, and \cite{rolet2016fast}.

\vspace{-0.3cm}
\subsection{Algorithm (dWDL)}

Given the previous estimate $(\Lambda_{n-1}, \mathcal{D}_{n-1})$, we compute the updated estimate $(\Lambda_{n}, \mathcal{D}_{n})$ by solving convex sub-problems as follows: 
\begin{align}
&\hspace{-0.3cm}\Lambda_{n}\in \argmin_{\Lambda \in \Sigma^{N}_{r}} \,\, f_W(\mathcal{D}_{n-1}, \Lambda) 
 + \frac{\tau_{n}}{2}\lVert \Lambda - \Lambda_{n-1} \rVert_{F}^{2} \label{alg:dWDL_high-level_Lambda} \\
&\hspace{-0.3cm}\mathcal{D}_{n}\in \argmin_{\mathcal{D}\in \Sigma^{r}_{I_{1}\times \cdots \times I_{d}}}
 \,\, f_W(\mathcal{D}, \Lambda_{n})   + \frac{\tau_{n}}{2}\lVert \mathcal{D} - \mathcal{D}_{n-1} \rVert_{F}^{2}.\label{alg:dWDL_high-level_D} 
\end{align}
For the standard nonnegative matrix factorization using the Frobenius norm instead of the Wasserstein norm, solving the corresponding convex sub-problems amounts to solving standard nonnegative least squares problem, which can be done by applying standard projected gradient descent. However, solving convex sub-problems in \eqref{alg:dWDL_high-level_Lambda} and \eqref{alg:dWDL_high-level_D}  is computationally demanding since one is required to compute $N$ Wasserstein distances $W_{\gamma}$, each of which involves finding an optimal transport plan by solving a separate optimization problem. Below, we propose a computationally efficient algorithm where one is only required to solve a single and simple subproblem (instead of $N$) for each block coordinate descent step. 
	
\begin{algorithm}[H]
\small
\caption{dWDL \eqref{eq:dWDL_main_problem}} 
\label{algorithm:dWDL_main}
\begin{algorithmic}[1]
    \STATE \textbf{Input:} $\param_{0}=(\mathcal{D}_0, \Lambda_0)\in \Sigma^{r}_{I_{1}\times \cdots \times I_{d}}\times \Sigma^{N}_{r}$ (initial estimate); $N$ (number of iterations);   $(\tau_{n})_{n\ge 1}$, (non-decreasing sequence in $[1,\infty)$);

    \STATE \quad \textbf{for} $n=1,\dots,N-1$ \textbf{do}: \STATE \quad \quad Update estimate $\param_{n-1}=(\mathcal{D}_{n-1}, \Lambda_{n-1})$ by  
    \begin{align}
        \Lambda_{n} &\leftarrow \textup{Algorithm \ref{algorithm:dWDL_Lambda} with input $(\mathcal{D}_{n-1}, \Lambda_{n-1} )$} \\
        \mathcal{D}_{n} & \leftarrow \textup{Algorithm \ref{algorithm:dWDL_D} with input $(\mathcal{D}_{n-1}, \Lambda_{n} )$} 
    \end{align}
    \STATE \quad \textbf{end for}
    \STATE \textbf{output:}  $\param_{N}$ 
\end{algorithmic}
\end{algorithm}

We now describe Algorithms \ref{algorithm:dWDL_Lambda} and \ref{algorithm:dWDL_D} that solve the convex sub-problems in \eqref{alg:dWDL_high-level_Lambda} and \eqref{alg:dWDL_high-level_D}.  To solve the primal problem \eqref{alg:dWDL_high-level_Lambda}, we consider its dual problem. For simplicity, denote the distance function and the proximal term by for $\X, y \in \Sigma_{I_{1}\times \cdots \times I_{d}}$ and for given $\lambda_0 \in \Sigma_{r}$,
\begin{align}
\nonumber
    H_{\X}(y) &:= W_{\gamma}\left(\X, \, y \right)
    \hbox{ and }\\
    F_{\lambda_0}(\lambda) &:= 
    \begin{cases}
         \frac{1}{2} \lVert \lambda - \lambda_0 \rVert_{F}^{2} &\hbox{ for } \lambda \in \Sigma_{r},\\
         +\infty &\hbox{ otherwise }.
    \end{cases} \label{eq:hf}
\end{align}
Then, the primal problem \eqref{alg:dWDL_high-level_Lambda} can be re-written as
\begin{align}
\nonumber
    \min_{ \Lambda \in \R^{r\times N} } \,\, \sum_{i=1}^{N} \{ H_{\X_i}  \left( \mathcal{D}_{n-1}\times_{d+1} \Lambda[:,i] \right) \\ + \tau_{n} F_{\Lambda_{n-1}[:,i]}(\Lambda[:,i])\}. \label{eq:primal}
\end{align}
Here, the condition $\Lambda \in \Sigma^{N}_{r}$ is enforced by $F$ in the second term.
	
Note that the above is a convex minimization problem but solving it directly is computationally expensive since simply evaluating the function $H_{\X_i}$ above involves finding an optimal transport map $T\in U(\vect(\A), \vect(\B))$. In order to overcome this issue, we consider the dual problem of \eqref{eq:primal} reminiscent of \citet{cuturi2013sinkhorn}. Introducing a dual variable $G \in \mathbb{R}^{I_{1}\times \cdots \times I_{d} \times N}$, we obtain the dual problem:
\begin{align}
\nonumber
& \min_{G \in \mathbb{R}^{I_{1}\times \cdots \times I_{d} \times N}} \sum_{i=1}^{N} \{ H_{\X_i}^*(-G[:,i]) \\ &+ \tau_n F_{\Lambda_{n-1}[:,i]}^*(\mathcal{D}_{n-1}\times_{\leq d} G[:,i]/\tau_n)  \}.\label{eq:dual}
\end{align}
Here, the \textit{conjugate} $f^{*}$ of $f$ is defined as
\begin{align}
    f^* : \mathbb{R}^d \rightarrow [-\infty, +\infty]: u \mapsto \sup_{x} (\langle x, u \rangle - f(x)).
\end{align}
This dual problem can be solved without having to deal with a matrix-scaling problem, as in the primal one (see \cite{cuturi2016smoothed}). 
We postpone further discussion about the conjugate functions $H^*$ and $F^*$ to the subsequent sections.
	
\begin{algorithm}[H]
\small
\caption{Solving for $\Lambda$} 
\label{algorithm:dWDL_Lambda}
\begin{algorithmic}[1]
    \STATE \textbf{Input:} $\param_{n-1}=(\mathcal{D}_{n-1}, \Lambda_{n-1}) \in \Sigma^{r}_{I_{1}\times \cdots \times I_{d}}\times \Sigma^{N}_{r}$ (current estimate);   $(\tau_{n})_{n\ge 1}$;
    \STATE \quad \quad Update estimate $\Lambda_{n-1}$ by  
     \begin{align*}
     & G^{\circ}_n \leftarrow \hbox{ the minimizer of \eqref{eq:dual}} \\
     & \Lambda_{n} \leftarrow \left(\Lambda_{n-1} + \frac{\mathcal{D}_{n-1}\times_{\leq d} G^{\circ}_n}{\tau_n} - J^\circ \otimes c^{\circ}_n \right)_+
    \end{align*}
    \quad \quad where $c^{\circ}_n \in \mathbb{R}^{N\times 1}$ is chosen to satisfy $\Lambda_n \in \Sigma^{N}_{r}$ and all entries of $J^\circ \in \R^{r \times 1}$ are one. 
    \STATE \textbf{output:}  $\param_{n-\frac{1}{2}} = (\mathcal{D}_{n-1}, \Lambda_n)$ 
\end{algorithmic}
\end{algorithm}

Here, the $1,2, \cdots, d$-mode product $\mathcal{D} \times_{\leq d} \Lambda$ of $\mathcal{D}\in \mathbb{R}^{I_{1}\times \dots \times I_{d} \times N}$ with a tensor $\Lambda \in \mathbb{R}^{I_1 \times I_{2} \times \cdots \times I_{d} \times J}$ is
\begin{align}
\nonumber
    &(\mathcal{D} \times_{\leq d} \Lambda) [j] := \sum_{i_1, i_2, \cdots, i_d} \mathcal{D}[i_1, i_2, \cdots, i_d] \nonumber \\  & \hspace{3cm} \times \Lambda[i_1, i_2, \cdots, i_d, j]. \label{eq:1nmode}
\end{align}
	
Based on similar arguments, the dual problem of \eqref{alg:dWDL_high-level_D} can be derived as follows:
\begin{align}
\nonumber
 & \min_{G \in \mathbb{R}^{I_{1}\times \cdots \times I_{d} \times N}} \left\{ \left( \sum_{i=1}^{N} H_{\X_i}^*(-G[:,i]) \right) \right. \\
&\left. + \tau_n F_{\mathcal{D}_{n-1}}^*(G \times_{d+1} \Lambda^{T}_n/\tau_n)  \right\}. \label{eq:dual2}
\end{align}

\begin{algorithm}[H]
\small
\caption{Solving for $\mathcal{D}$} 
\label{algorithm:dWDL_D}
\begin{algorithmic}[1]
\STATE \textbf{Input:} $\param_{n-\frac{1}{2}}=(\mathcal{D}_{n-1}, \Lambda_{n}) \in \Sigma^{r}_{I_{1}\times \cdots \times I_{d}}\times \Sigma^{N}_{r}$ (current estimate);   $(\tau_{n})_{n\ge 1}$;
\STATE \quad \quad Update estimate $\mathcal{D}_{n-1}$ by  
 \begin{align*}
 & G^{\dagger}_n \leftarrow \hbox{ the minimizer of \eqref{eq:dual2}}\\
 & \mathcal{D}_{n} \leftarrow \left(\mathcal{D}_{n-1} + \frac{G^{\dagger}_n \times_{d+1} \Lambda^{T}_n}{\tau_n}- J^{\dagger} \otimes c^{\dagger}_n \right)_+
\end{align*}
\quad \quad where $c^{\dagger}_n \in \mathbb{R}^{r\times 1}$ is chosen to satisfy $\mathcal{D}_{n} \in \Sigma^{r}_{I_{1}\times \cdots \times I_{d}}
$ and all entries of $J^\dagger \in \R^{I_{1}I_{2} \cdots I_{d} \times 1}$ are one.
\STATE \textbf{output:}  $\param_{n} = (\mathcal{D}_{n}, \Lambda_n)$ 
\end{algorithmic}
\end{algorithm}
	
The per-iteration cost of Algorithms~\ref{algorithm:dWDL_Lambda} and \ref{algorithm:dWDL_D} is given by $O((I_{1}\dots I_{d})^{2} N)$.

\section{Theoretical guarantees of Wasserstein dictionary learning}

We prove that our computationally efficient algorithm, Algorithm \ref{algorithm:dWDL_main}, is actually solving BCD with proximal regularization for our main problem \eqref{eq:dWDL_main_problem}. The proof of Theorem~\ref{lem:per_iter_correctness} can be found in Appendix~\ref{ap:per_iter_correctness}.

\begin{theorem}(Per-iteration correctness)
\label{lem:per_iter_correctness}
     Algorithm \ref{algorithm:dWDL_main} solves \eqref{alg:dWDL_high-level_Lambda} and \eqref{alg:dWDL_high-level_D}. 
\end{theorem} 

Formally speaking, the dual problem \eqref{eq:dual} is derived from the primal problem \eqref{eq:primal} as follows: for given $(\mathcal{D}_{n-1}, \Lambda_{n-1}) \in \Sigma^{r}_{I_{1}\times \cdots \times I_{d}}\times \Sigma^{N}_{r}$ and $\tau_n > 0$,
\begin{align*}
    &\min_{\Lambda \in \Sigma_{r}^{N} }  H_{\X_i}  \left( \mathcal{D}_{n-1}\times_{d+1} \Lambda[:,i] \right) + \tau_{n} F_{\Lambda_{n-1}[:,i]}(\Lambda[:,i]),\\
    & = \min_{\substack{\Lambda \in \Sigma_{r}^{N},\\ Q \in \Sigma^{N}_{I_{1}\times \cdots \times I_{d}}}} \max_{G \in  \mathbb{R}^{I_{1}\times \cdots \times I_{d} \times N}}  H_{\X_i}  \left( Q[:,i] \right) \\
    &\qquad + \tau_{n} F_{\Lambda_{n-1}[:,i]}(\Lambda[:,i]) \\ &\qquad + \langle Q[:,i] - \mathcal{D}_{n-1} \times_{d+1} \Lambda[:,i], G[:,i] \rangle,\\
    &  = - \min_{G \in \mathbb{R}^{I_{1}\times \cdots \times I_{d} \times N}}  H_{\X_i}^*(-G[:,i]) \\
    &\qquad + \tau_n F_{\Lambda_{n-1}[:,i]}^*(\mathcal{D}_{n-1}\times_{\leq d} G_n[:,i]/\tau_n).
\end{align*}
The above derivation is standard in the classical theory of convex optimization. However, solving Algorithm~\ref{algorithm:dWDL_main} requires us to find the optimizers of the primal problem \eqref{alg:dWDL_high-level_Lambda} and \eqref{alg:dWDL_high-level_D} in terms of the inputs and their dual solutions. Due to the constraints, $\mathcal{D} \in \Sigma^{r}_{I_{1}\times \cdots \times I_{d}}$ and $\Lambda \in \Sigma^{N}_{r}$, this does not directly follows.

To establish the correctness rigorously, we consider a general minimization problem of a bivariate function under inequality constraints in Lemma~\ref{lem:saddle}: for given functions $f: \mathcal{K} \rightarrow (-\infty, +\infty]$, $h: \mathcal{H} \rightarrow (-\infty, +\infty]$, and $R:\mathcal{H} \rightarrow \mathcal{K}$,
\begin{align}
\label{eq:gprimal}
    \min_{x \in \mathcal{H}, Rx \in K} f(Rx) + h(x).
\end{align}
Here, $\mathcal{H}$ and $\mathcal{K}$ are real Hilbert spaces with inner product $\langle\cdot, \cdot \rangle$, and $K$ is a nonempty closed convex cone in $\mathcal{K}$. The key idea is based on Propositions 19.18 and 19.23 in \citet{bauschke2011convex}, but we provide the proof in Appendix~\ref{ap:per_iter_correctness} for the sake of completeness.

Now we can obtain a convergence and complexity result for Algorithm~\ref{algorithm:dWDL_main} using Theorems~\ref{lem:per_iter_correctness} and\ref{thm:BCD}.

\begin{theorem}\label{thm:BCD_dWDL}
Suppose that Assumption~\ref{assumption:A3} holds, the proximal regularizers $(\tau_{n})_{n\ge 1}$ satisfy $\tau_{n}> 1/\gamma$ for $n\ge 1$ and $\tau_{n}=O(1)$. For a output $(\param_{n})_{n\ge 0}$ of Algorithm \ref{algorithm:dWDL_main}, the following hold: 
\begin{description}
\item[(i)] (Global convergence to stationary points) Every limit point of $(\param_{n})_{n\ge 0}$ is a stationary point of $f_W$ over $\Param:= \Sigma^{r}_{I_{1}\times \cdots \times I_{d}} \times \Sigma^{N} _{r}$.

\vspace{0.1cm}
\item[(ii)]  (Worst-case rate of convergence)  There exists a constant $M$ independent of $\param_{0}$ such that for  $n\ge 1$, \eqref{eq:thm_convergence_bd_inexact} in Theorem \ref{thm:BCD} holds.

\vspace{0.1cm}
\item[(iii)] (Worst-case complexity) 
The worst-case iteration complexity $N_{\eps}$ for Algorithm \ref{algorithm:dWDL_main} satisfies $N_{\eps} = O(\eps^{-1} (\log \eps^{-1})^{2} )$. Furthermore, the worst-case complexity of Algorithm \ref{algorithm:dWDL_main} is 
\begin{align*}
&O(N_{\eps}\cdot \textup{(worst-case cost of solving sub-problems)})  \\
&\qquad = O(N_{\eps} \cdot \log N_{\eps} \cdot \textup{(cost of PGD step for dual)}) \\
& \qquad = O( \eps^{-1} (\log \eps^{-1})^{3} (I_{1}\times \dots \times I_{d})^{2}N ).
\end{align*}
\end{description}
\end{theorem}

\begin{proof}[\textbf{Proof of Theorem~\ref{thm:BCD_dWDL}}]

Let us first show that Algorithm \ref{algorithm:dWDL_main} satisfies Assumptions~ 
\ref{assumption:A1}, and \ref{assumption:A2}. Then, (i) and (ii) follow from Theorem~\ref{thm:BCD}.
The conjugate function of $H_\X$ given in \eqref{eq:hf} 
has a closed form \cite{cuturi2016smoothed}: for $g \in \mathbb{R}^{I_{1}\times \cdots \times I_{d}}$ and given $\X \in \Sigma_{I_{1}\times \cdots \times I_{d}}$,
\begin{align*}
    H^*_\X(g;\Sigma_{I_{1}\times \cdots \times I_{d}}) &:= \sup_{y \in \Sigma_{I_{1}\times \cdots \times I_{d}}} \langle g, y \rangle - H_\X(y),\\
    &= \gamma \left( \langle \X, \log \X \rangle + \langle \X, \ \ \log ( K \alpha ) \rangle \right).
\end{align*}
Here, $K = \exp(-M/\gamma) \in \left(\mathbb{R}^{I_{1}\times \cdots \times I_{d}}\right)^2$, $\alpha = \exp (g/\gamma) \in \mathbb{R}^{I_{1}\times \cdots \times I_{d}}$, and $M \in \left(\mathbb{R}^{I_{1}\times \cdots \times I_{d}}\right)^2$ is a given cost matrix. It is known from Theorem 2.4 in \citet{cuturi2016smoothed} that this dual function is $C^\infty$. In addition, its gradient function is $1/\gamma$ Lipschitz, and it is explicitly given as
\begin{align}
\label{eq:gradh}
    \nabla H^*_\X(g) = \alpha \circ \left( K \frac{\X}{K \alpha} \right) \in \Sigma_{I_{1}\times \cdots \times I_{d}}.
\end{align}
Therefore, Assumption~\ref{assumption:A1} is satisfied. Furthermore, the constraint set $\Sigma^{r}_{I_{1}\times \cdots \times I_{d}}$ and $\Sigma^{N}_{r}$ satisfy Assumption~\ref{assumption:A2}.

Next, we compute the per-iteration cost of Algorithms~\ref{algorithm:dWDL_Lambda} and \ref{algorithm:dWDL_D}. 
The dual function of $F_{\lambda_0}$ is given by for $g \in \mathbb{R}^r$
\begin{align*}
    F_{\lambda_0}^*(g) &:= \sup_{\lambda \in \Sigma_{r}} \langle g, \lambda \rangle - \frac{1}{2}\lVert \lambda - \lambda_0 \rVert_{F}^{2}.
\end{align*}
From Lemma~\ref{lem:fdual}, the optimizer of the above
is given as
\begin{align}
\label{eq:fdual2}
    \lambda^* = (g + \lambda_0 - c 1_{r})_+
\end{align}
where $c$ is a constant chosen to satisfy $\lambda \in \Sigma_{r}$, and thus
\begin{align*}
    F_{\lambda_0}^*(g) = \frac{1}{2} (g + \lambda_0 - c 1_{r})_+ (g + \lambda_0 + c 1_{r}) - \frac{1}{2} \lVert \lambda_0 \rVert_{F}^{2}.
\end{align*}
By the duality as in Lem. 7.15 in \citet{santambrogio2015optimal}, its gradient is given as the optimizer \eqref{eq:fdual2}: 
$\nabla F_{\lambda_0}^*(g) = \lambda = (g + \lambda_0 - c 1_{r})_+ \in \Sigma_r$. Therefore, each gradient descent step to solve \eqref{eq:dual} or \eqref{eq:dual2} requires $O((I_{1}\dots I_{d})^{2} N)$. Lastly, \eqref{eq:dual} and \eqref{eq:dual2} are convex problems, we conclude (iii).
\end{proof}

\section{Extension to Wasserstein CP-dictionary learning}

While it is possible to vectorize general $d$-mode tensor-valued signals to reduce to the case of dictionary learning for vector-valued signals, it would be more beneficial to tailor the $d$-dimensional dictionary learning problem \eqref{eq:tensor_DL_2} to exploit particular tensor structures that one desires to respect. One such approach is to constrain further the type of dictionary atoms $\D_{1},\dots,\D_{r}$ that we allow. Namely, the CONDECOMP/PARAFAC (CP)-dictionary learning \cite{lyu2020online_tensor} assumes that each $\D_{i}$ is a rank-1 tensor in the sense that it is the outer product of some 1-dimensional vectors. Also, exploiting Tucker-decomposition structure on the dictionary atoms has been studied recently in \citet{shakeri2016minimax, ghassemi2017stark}. 

\subsection{Wasserstein CP-dictionary learning}

Suppose a data tensor $\X \in  \R^{I_{1}\times \cdots \times I_{d}}$ is given and fix an integer $r \ge 1$. In the  CANDECOMP/PARAFAC (CP) decomposition of $\X$ \cite{kolda2009tensor}, we would like to find $r$ \textit{loading matrices} $U^{(i)} \in \R^{I_{i}\times r}$ for $i=1,\dots, d$ such that the sum of the outer products of their respective columns approximate $\X$:
\begin{align*}
    \X \approx  \sum_{k=1}^{r} \bigotimes_{i=1}^{d} U^{(i)}[:,k] =: [\![ U^{(1)},U^{2},\dots,U^{(d)}  ]\!] 
\end{align*}
where $U^{(i)}[:,k]$ denotes the $k^{\textup{th}}$ column of the $I_{i}\times r$ loading matrix matrix $U^{(i)}$ and $\bigotimes$ denotes the outer product. We have also introduced the bracket operation $[\![ \cdot ]\!]$.  
	
As an optimization problem, the above CP decomposition model can be formulated as the following the \textit{constrained CP-decomposition} problem: 
\begin{align}\label{eq:NTF} 
    \argmin_{U^{(1)}\in \Theta^{(1)},\dots, U^{(d)}\in \Theta^{(d)}} f_{\textup{CP}}(U^{(1)},\dots, U^{(d)})
\end{align}
where 
\begin{align*}
    f_{\textup{CP}}(U^{(1)},\dots, U^{(d)}) :=  \left\lVert \X - [\![ U^{(1)},U^{2},\dots,U^{(d)}  ]\!] \right\rVert_{F}^{2}
\end{align*}
and $\Theta^{(i)}\subseteq \R^{I_{i}\times r}$ denotes a compact and convex constraint set and $\lambda_{i}\ge 0$ is a $\ell_{1}$-regularizer for the $i^{\textup{th}}$ loading matrix $U^{(i)}$ for $i=1,\dots, d$.  In particular, by taking  $\lambda_{i}=0$ and  $\Theta^{(i)}$ to be the set of nonnegative $I_{i}\times r$ matrices with bounded norm for $i=1,\dots, d$,  \eqref{eq:NTF} reduces to the \textit{nonnegative CP decomposition} (NCPD) \cite{shashua2005non, zafeiriou2009algorithms}. Also, it is easy to see that $f_{\textup{CP}}$ is equal to 
\begin{align}\label{eq:NTF_dict}
    \left\lVert \X - \Out(U^{(1)}, \dots, U^{(d-1)}) \times_{d} (U^{(d)})^{T} \right\rVert_{F}^{2},
\end{align}
which is the \textit{CP-dictionary-learning} problem introduced in \citet{lyu2020online_tensor}. Here 
$\times _{d}$ denotes the mode-$d$ product (see \cite{kolda2009tensor}) the outer product of loading matrices $U^{(1)},\dots, U^{(m)}$ is defined as 
\begin{align}
    \nonumber
    &\Out(U^{(1)},\dots,U^{(d)}) := \\ &\left[\bigotimes_{k=1}^{d} U^{(k)}[:,1],\, \bigotimes_{k=1}^{d} U^{(1)}[:,2],\, \dots \,, \bigotimes_{k=1}^{d} U^{(k)}[:,r]  \right] %\in \R^{I_{1}\times \dots \times I_{d} \times r}. 
    \label{eq:def_out}
\end{align}
Namely, we can think of the $d$-mode tensor $\mathbf{X}$ as $I_{d}$ observations of $(d-1)$-mode tensors, and the $R$ rank-1 tensors in $\Out(U^{(1)}, \dots , U^{(d)})$ serve as dictionary atoms, whereas the transpose of the last loading matrix $U^{(d)}$ can be regarded as the code matrix. 
	
The Wasserstein formulation of the \textit{CP-dictionary-learning} problem \eqref{eq:NTF} is given as follows. As in the setting of \eqref{eq:dWDL_main_problem}, we suppose that each $d$-mode tensor $\X_i$ is a $d$-dimensional joint distribution. We aim to represent each data tensor $X_i$ based on the product distributions of $d$ one-dimensional distributions, $U^{(i)} \in \Sigma_{I_i}^r$ for $i = 1, \cdots, d$:
\begin{align}
[\mathbf{X}_{1},\dots, \mathbf{X}_{N}] \approx \Out(U^{(1)},\dots,U^{(d)}) \times_{d+1} \Lambda
\end{align}
for some code matrix $\Lambda \in \Sigma_{r}^N$ where $\Out$ is given in \eqref{eq:def_out}. Comparing the Wasserstein distance between each $X_i$ and the corresponding distribution, we formulate our main problem of Wasserstein CP-dictionary Learning (WCPDL):
\begin{align}\label{eq:WCPDL} 
\argmin_{\substack{U^{(1)}\in \Sigma_{I_{1}}^r,\dots, U^{(d)}\in \Sigma_{I_{d}}^r,\\ \Lambda\in \Sigma_{r}^{N} }} f_{\textup{WCP}}(U^{(1)},\dots, U^{(d)}, \Lambda)
\end{align}
where 
\begin{align*}
&f_{\textup{WCP}}(U^{(1)},\dots, U^{(d)}, \Lambda) \\&:= \sum_{i=1}^{N} W_{\gamma} \left( \X_i, \, \Out(U^{(1)},\dots,U^{(d)}) \times_{d+1} \Lambda [:,i] \right).  
\end{align*}

\subsection{Algorithm (WCPDL)}

We state our algorithm to solve Wasserstein CP-dictionary Learning \eqref{eq:WCPDL}. Given the previous estimates $U_{n-1}^{(1)},\dots, U_{n-1}^{(d)}$ and $\Lambda_{n-1}$, we compute the updated estimate $U_{n}^{(1)},\dots, U_{n}^{(d)}$ and $\Lambda_{n}$ by solving convex sub-problems, iteratively, as follows. 

\begin{algorithm}[t]
    \small
    \caption{WCPDL \eqref{eq:WCPDL}} 
    \label{algorithm:WCPDL_main}
    \begin{algorithmic}[1]
        \STATE \textbf{Input:} $\param_{0}=(U_{0}^{(1)},\dots, U_{0}^{(d)}, \Lambda_0)\in \Sigma^{r}_{I_{1}}\times \cdots \times \Sigma^{r}_{I_{d}}\times \Sigma^{N}_{r}$ (initial estimate); $N$ (number of iterations);   $(\tau_{n})_{n\ge 1}$, (non-decreasing sequence in $[1,\infty)$);
        \STATE \quad \textbf{for} $n=1,\dots,N-1$ \textbf{do}: \STATE \quad \quad Update estimate $\param_{n-1}=(U_{n-1}^{(1)},\dots, U_{n-1}^{(d)}, \Lambda_{n-1})$ by  
        \begin{align*}
            &\mathcal{D} \leftarrow \Out(U_{n-1}^{(1)},\dots,U_{n-1}^{(d)}) \\
            &\Lambda_{n} \leftarrow \textup{Output of Algorithm \ref{algorithm:dWDL_Lambda} with input $(\mathcal{D}, \Lambda_{n-1})$ %$\param_{n-1}$
            }; %$(U_{n-1}^{(1)},\dots, U_{n-1}^{(d)}, \Lambda_{n-1} )$};
        \end{align*}
%			\STATE \quad \qquad \textbf{end for}
        \STATE \quad \quad  \textbf{for} $k=1,\dots,d$ \textbf{do}: \STATE \quad \quad \quad
        Update estimate $U_{n-1}^{(k)}$ by
\begin{align*}
& \overline{\Lambda} \leftarrow \Out(U_{n}^{(1)},\dots,U_{n}^{(k-1)},U_{n-1}^{(k+1)}, \dots, U_{n-1}^{(d)}, \Lambda_n^T) \\
& \overline{\Lambda} \leftarrow \hbox{ Inserting the last mode of $\overline{\Lambda}$ into the $k$th mode} \\
&U_{n}^{(k)} \leftarrow \textup{Output of Algorithm \ref{algorithm:dWDL_D} with input
$(U_{n-1}^{(k)}, \overline{\Lambda})$
}  
\end{align*}
        \STATE \quad \quad \textbf{end for}
        \STATE \quad \textbf{end for}
        \STATE \textbf{output:}  $\param_{N}$ 
    \end{algorithmic}
\end{algorithm}

First, let $\mathcal{D}_{n-1}$ be $\Out(U_{n-1}^{(1)},\dots,U_{n-1}^{(d)}) \in \Sigma_{I_1 \times I_2 \times \cdots \times I_d}^r$. For a given data tensor $\X \in \Sigma_{I_1 \times I_2 \times \cdots \times I_d}^{N}$, $\tau_n>0$, and the previous estimates above, the code matrix is updated as follows:
\begin{align}
\nonumber
\Lambda_{n}\in \argmin_{\Lambda \in \Sigma^{N}_{r}} \,\, &\left( \sum_{i=1}^{N}W_{\gamma}\left(\X_i, \, (\mathcal{D}_{n-1} \times_{d+1} \Lambda) [:,i] \right)\right) \\ &+ \frac{\tau_{n}}{2}\lVert \Lambda - \Lambda_{n-1} \rVert_{F}^{2}. \label{alg:WCPDL_high-level_Lambda}
\end{align}

Next, for each $k \in \{1,2, \cdots, d\}$, let $\overline{\Lambda} \in \mathbb{R}^{I_1 \times I_2 \times \cdots \times I_{k-1} \times r \times  I_{k+1} \times \cdots \times I_{d} \times N}$ be obtained from 
\begin{align*}
    \Out(U_{n}^{(1)},\dots,U_{n}^{(k-1)},U_{n-1}^{(k+1)}, \dots, U_{n-1}^{(d)}, \Lambda_n^T)
\end{align*}
in $\mathbb{R}^{I_1 \times I_2 \times \cdots \times I_{k-1} \times I_{k+1} \times \cdots \times I_{d} \times N \times r}$ by inserting the last mode into the $k$th mode. Given $\overline{\Lambda}$, the dictionaries are updated as follows: 
\begin{align}
\nonumber
U_{n}^{(k)}\in \argmin_{U \in \Sigma_{I_{k}}^r
} \,\, &\left( \sum_{i=1}^{N}W_{\gamma}\left(\X_i, \, 
\overline{\Lambda}[:,i] \times_k U^T \right) \right)\\
  &+ \frac{\tau_{n}}{2}\lVert U^{(k)} - U_{n-1}^{(k)} \rVert_{F}^{2}. \label{alg:WCPDL_high-level_D} 
\end{align}

\begin{theorem}(Per-iteration correctness)
\label{lem:per_iter_correctness_WCPDL}
     Algorithm \ref{algorithm:WCPDL_main} solves \eqref{alg:WCPDL_high-level_Lambda} and \eqref{alg:WCPDL_high-level_D}. 
\end{theorem}

\section{Experiments}

\subsection{Wasserstein barycenter problem}

We first provide the simplest example when $r=1$. In this case, $\Lambda \in \Sigma_1^N$ and thus all entries of $\Lambda$ are 1's, which corresponds to the Wasserstein barycenter problem with equal weights:
$\min_{\D\in \Sigma_{I_{1},\dots,I_{d}}} \sum_{i=1}^{N}W_{\gamma}\left(\X_i, \, \D \right).$

For data living in the space of probability distributions, using the Wasserstein metric instead of the Euclidean metric may provide a better representation. Figure~\ref{fig:1dwb} provides the barycenter with respect to Wasserstein distance and the Frobenius norm when $d=1$, $r=1$, and $N=3$.

As shown in the figure, the Wasserstein barycenter of three Gaussian distributions is close to the Gaussian distribution, while the Frobenius one is given as the vertical average of three distributions, which shows a significant difference between the two formulations. 

\begin{figure}[ht]
\centering
\begin{subfigure}[b]{0.47\columnwidth}
\centerline{\includegraphics[width=\columnwidth]{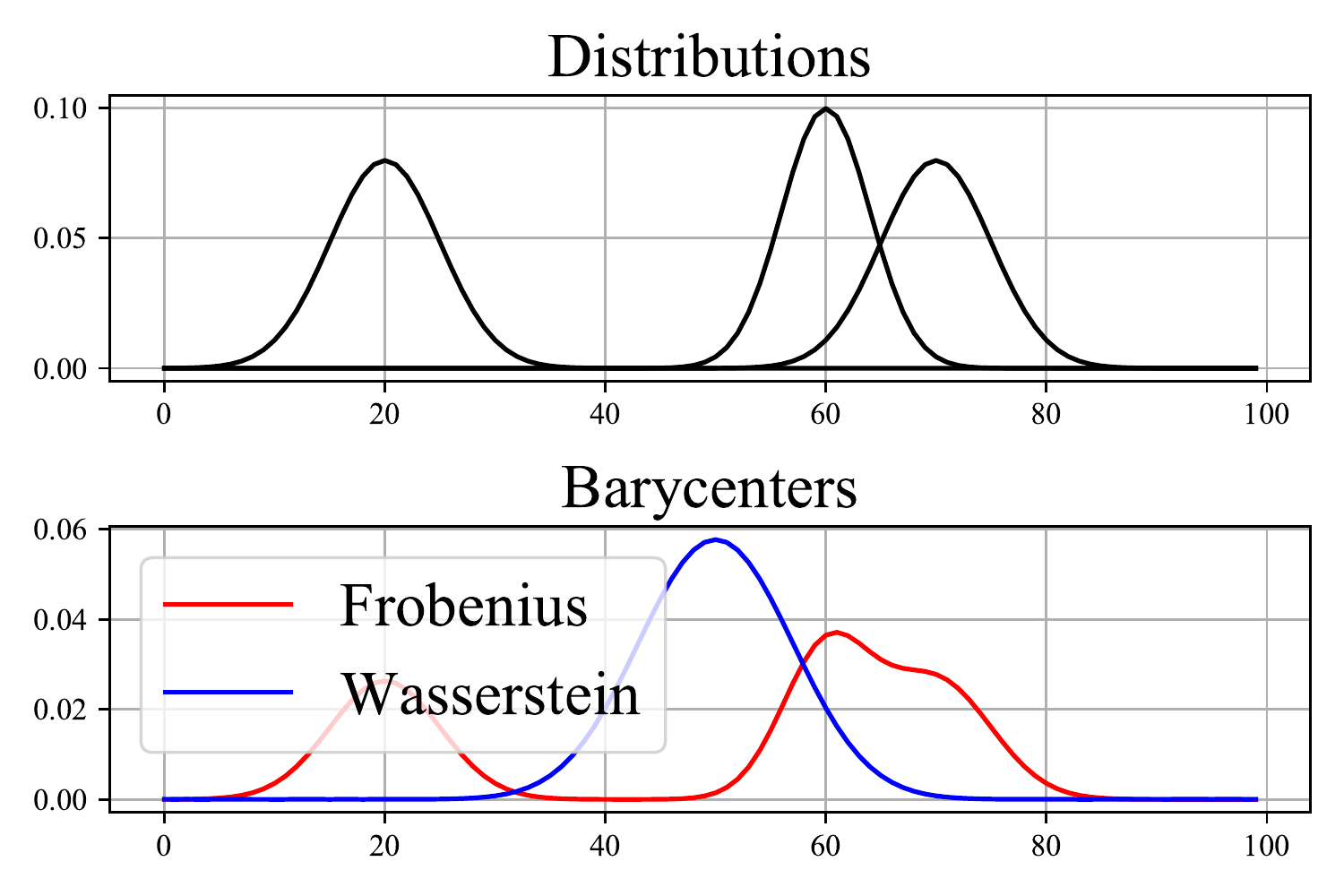}}
\end{subfigure}
\hfill
\begin{subfigure}[b]{0.45\columnwidth}
\centerline{\includegraphics[width=\columnwidth, trim=0cm 0cm 1cm 1cm]{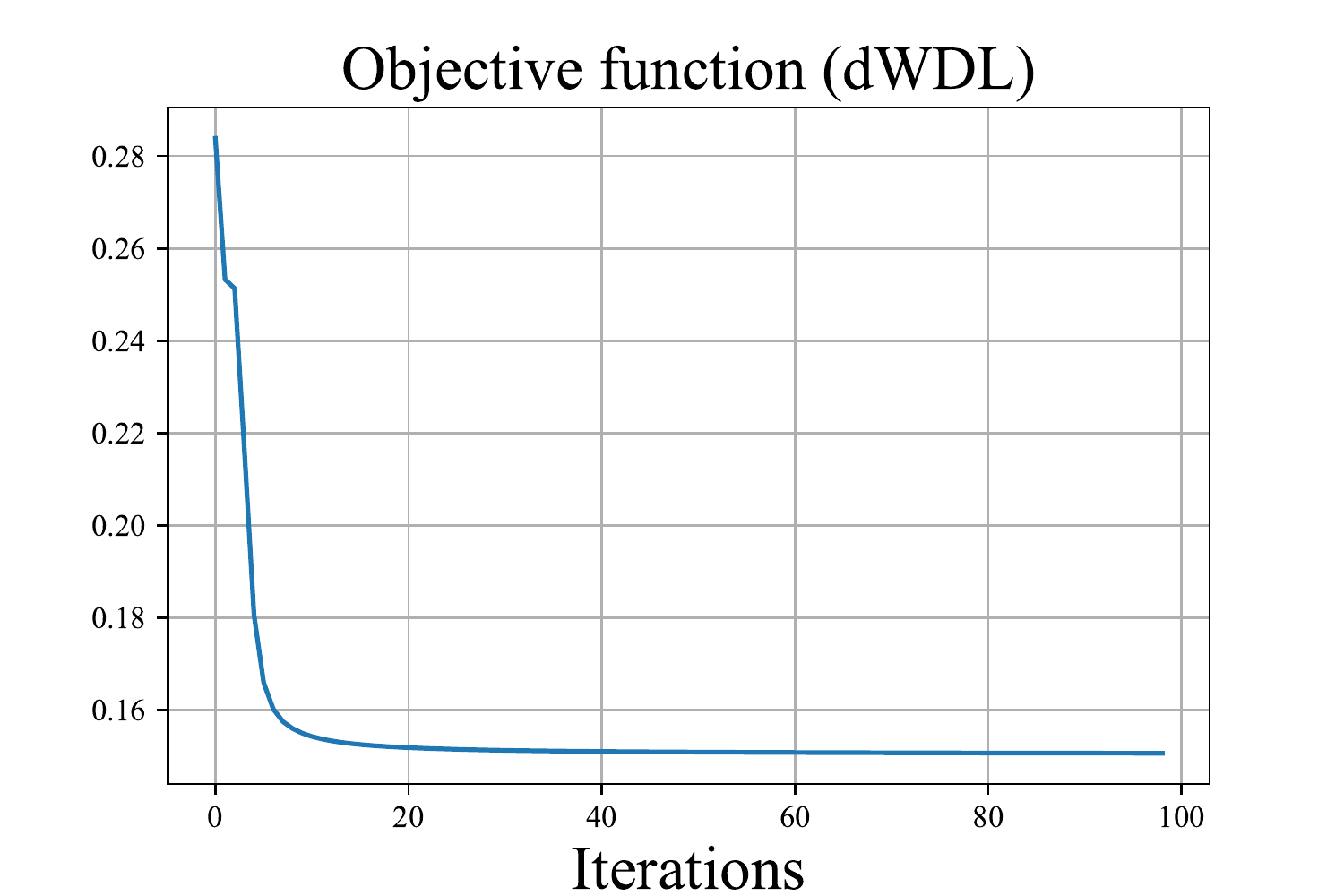}}
\end{subfigure}
\caption{Finding the barycenter of three Gaussian distributions with respect to Wasserstein distance and the Frobenius norm}
\label{fig:1dwb}
\end{figure}

\begin{figure}[ht]
\centering
\begin{subfigure}[b]{0.47\columnwidth}
\centerline{\includegraphics[width=\columnwidth]{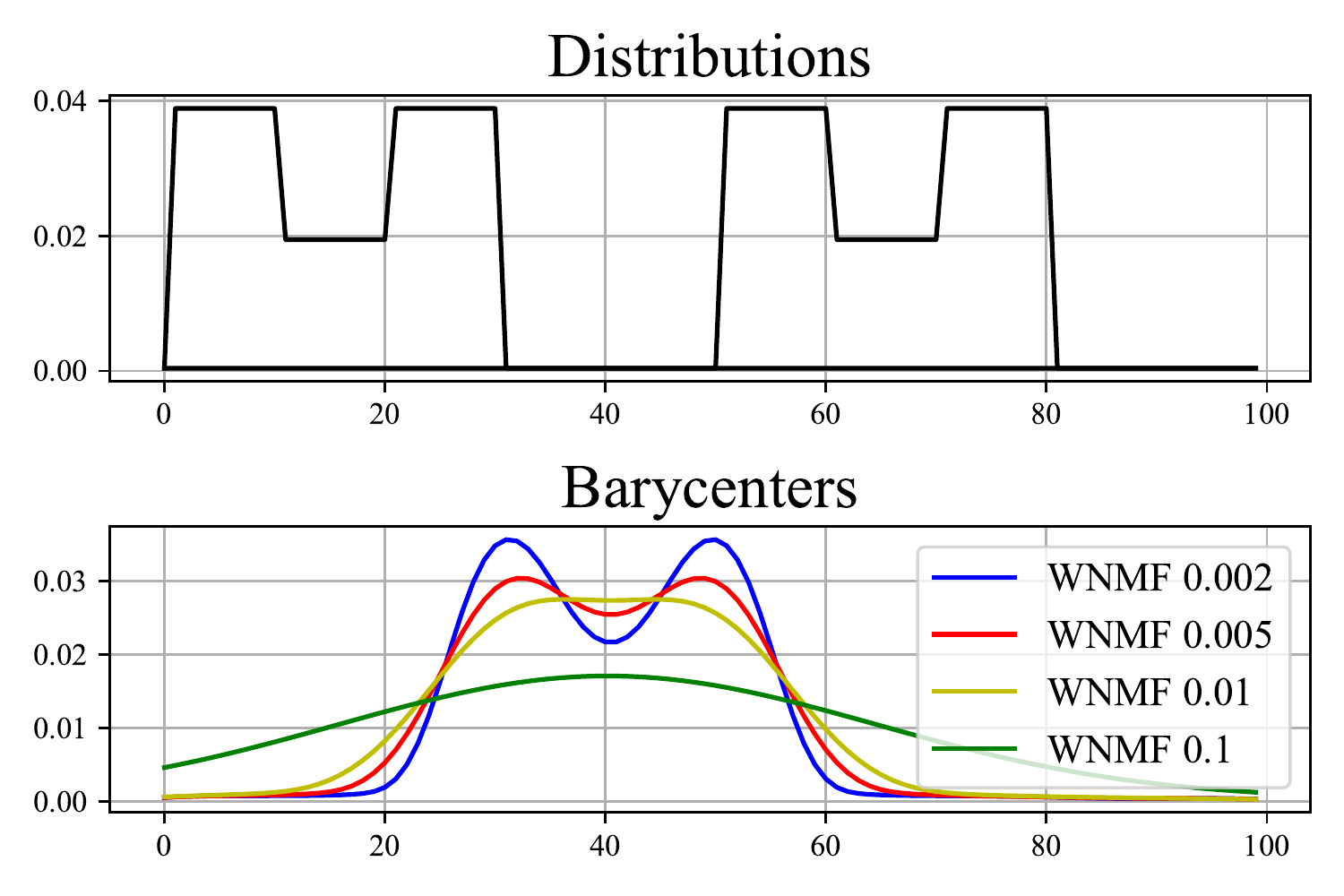}}
\end{subfigure}
\hfill
\begin{subfigure}[b]{0.45\columnwidth}
\centerline{\includegraphics[width=\columnwidth, trim=0cm 0cm 1cm 1cm]{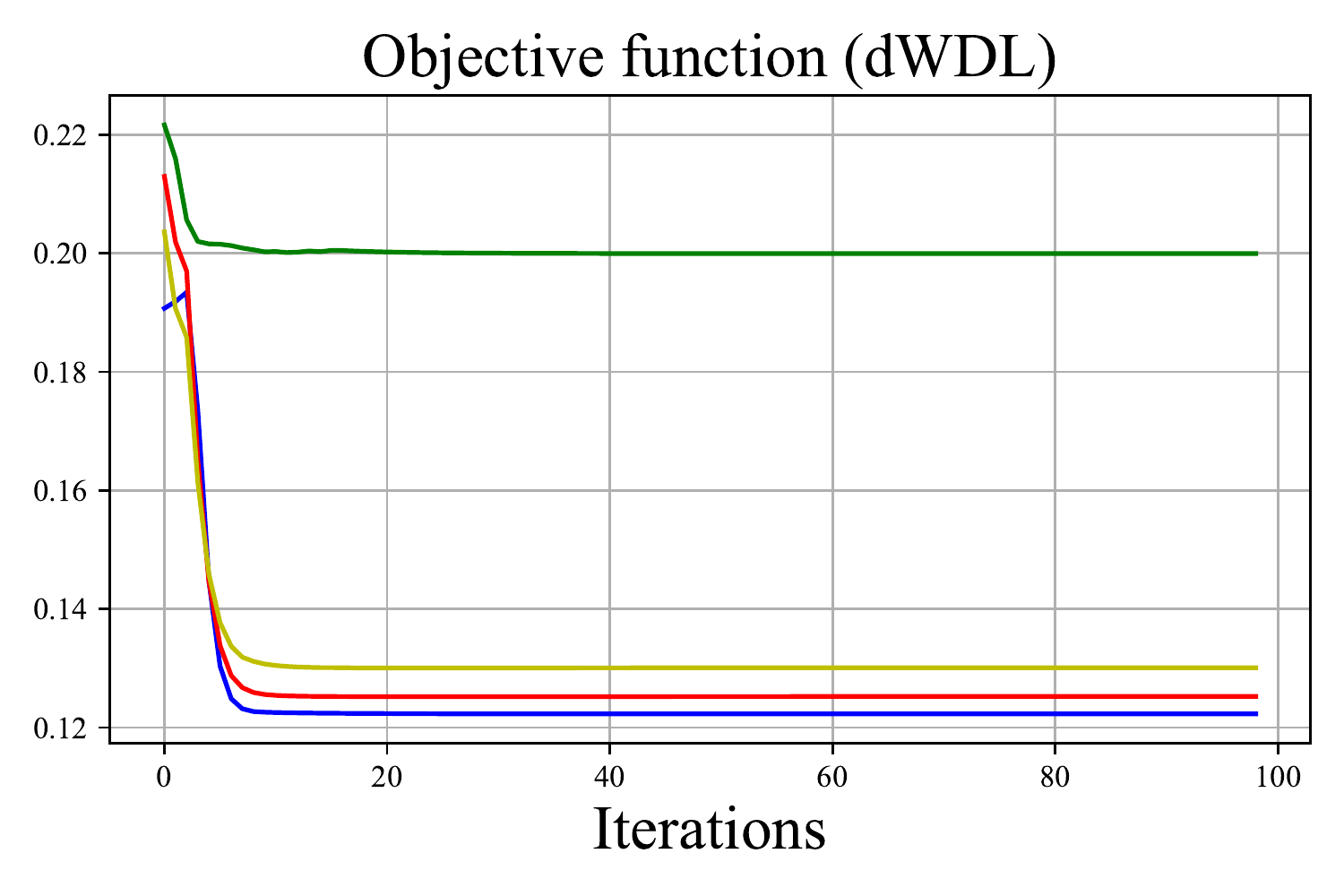}}
\end{subfigure}
\caption{Finding the barycenter of two $\sqcup$-shaped distributions with respect to Wasserstein distance for different $\gamma$'s}
\label{fig:1dwb_b}
\end{figure}

As defined in \eqref{eq:Wasserstein_def}, the regularized Wasserstein distance $W_{\gamma}$ depends on the parameter $\gamma>0$. In Figure~\ref{fig:1dwb_b}, we solve the Wasserstein barycenter problem for different $\gamma$'s and two $\sqcup$-shaped distributions. While two peaks  appear in $\gamma = 0.002$ and $\gamma = 0.005$, the distribution is getting close to Gaussian. This illustrates the importance of choosing appropriate $\gamma$ to find out the geometric property of data sets.

\subsection{Wasserstein dictionary learning}

\begin{figure}[ht]
\centering
\begin{subfigure}[b]{0.45\columnwidth}
\centerline{\includegraphics[width=\columnwidth]{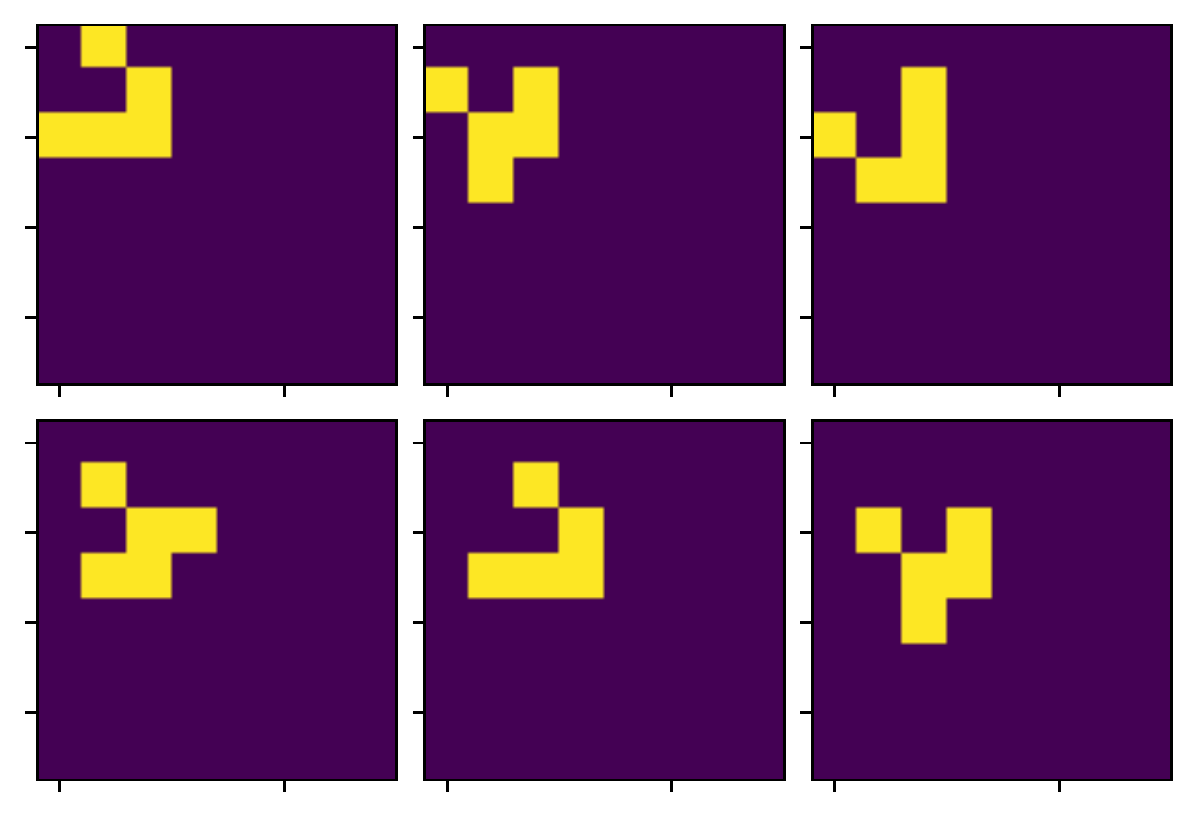}}
\end{subfigure}
\hfill
\begin{subfigure}[b]{0.38\columnwidth}
\centerline{\includegraphics[width=\columnwidth]{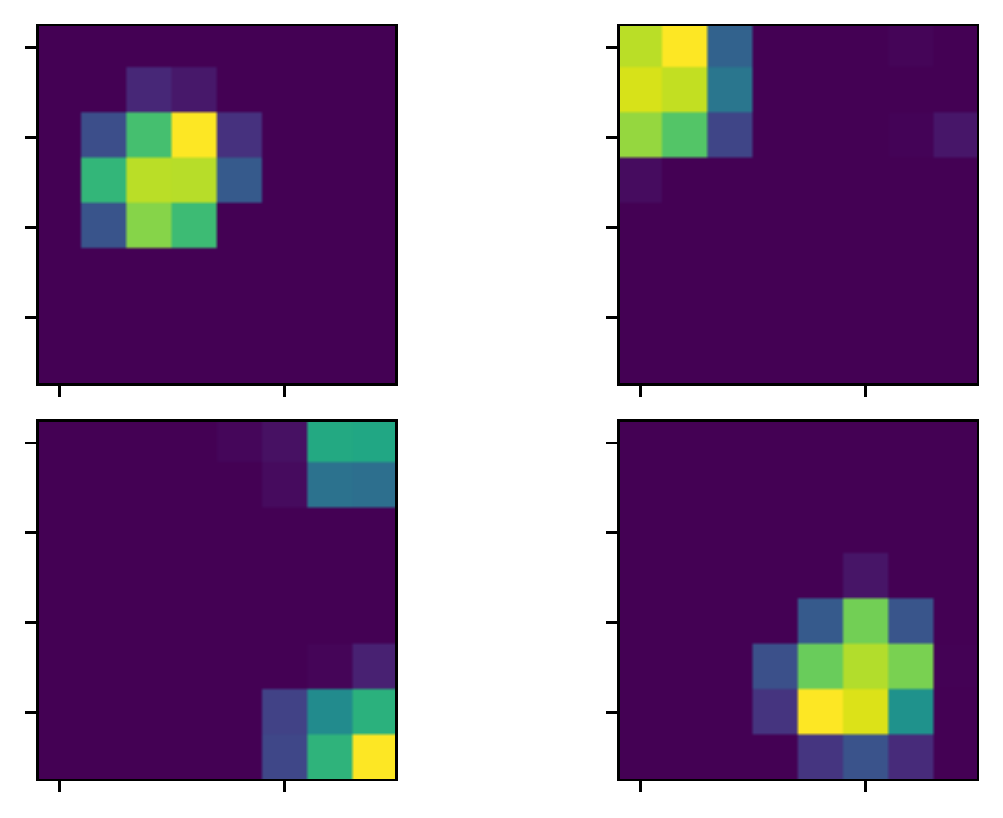}}
\end{subfigure}
\hfill
\begin{subfigure}[b]{0.09\columnwidth}
\centerline{\includegraphics[width=\columnwidth]{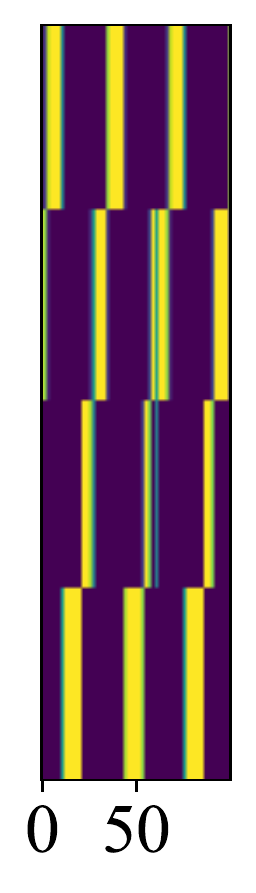}}
\end{subfigure}

\caption{Wasserstein dictionary learning with $r=4$, $N=100$, and the Euclidean distance; a sequence of images (left), dictionaries (middle), code matrices (right)}
\label{fig:game1}
\end{figure}

\begin{figure}[ht]
\centering
\begin{subfigure}[b]{0.38\columnwidth}
\centerline{\includegraphics[width=\columnwidth]{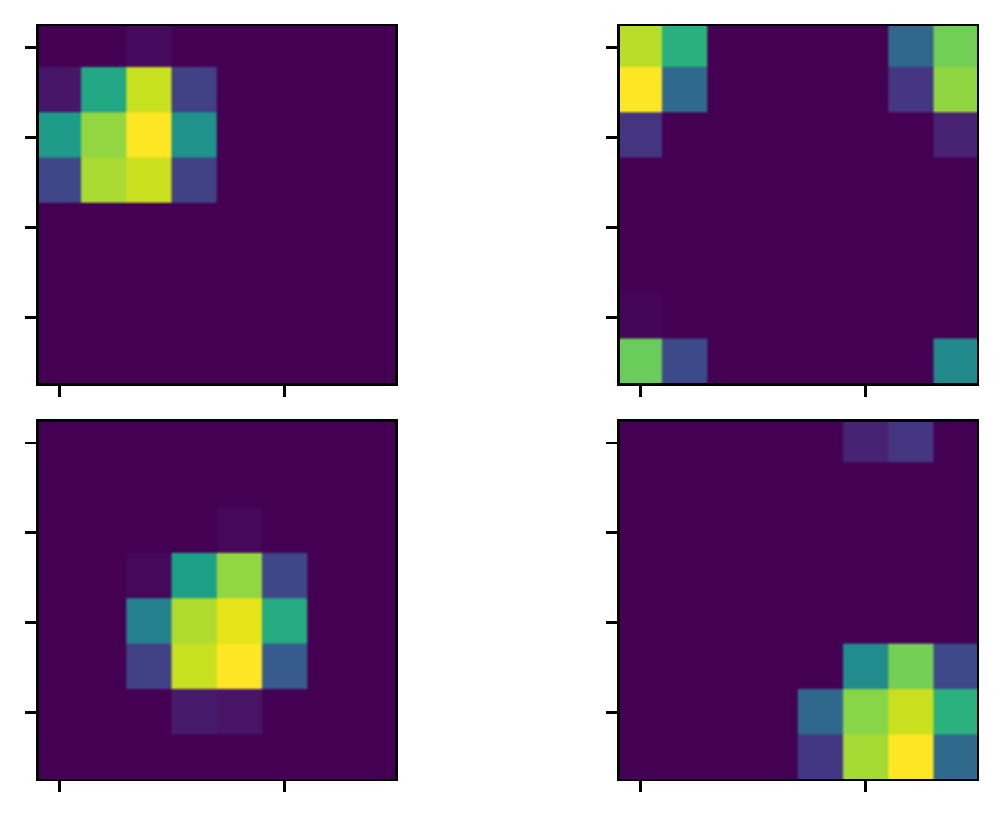}}
\end{subfigure}
\hfill
\begin{subfigure}[b]{0.3\columnwidth}
\centerline{\includegraphics[width=\columnwidth]{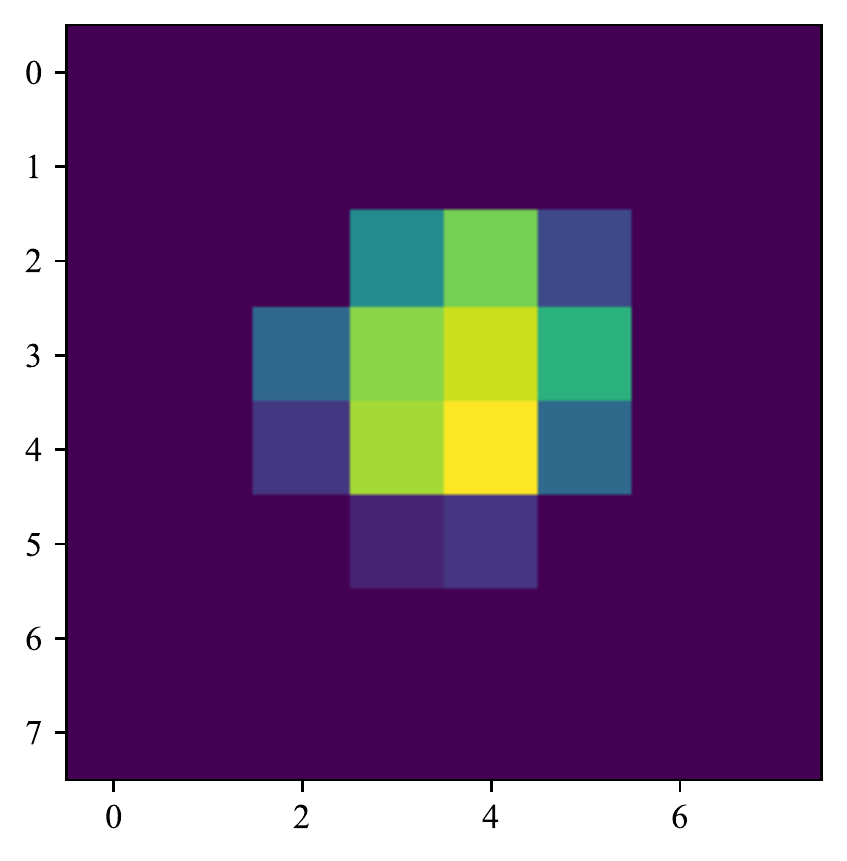}}
\end{subfigure}
\hfill
\begin{subfigure}[b]{0.09\columnwidth}
\centerline{\includegraphics[width=\columnwidth]{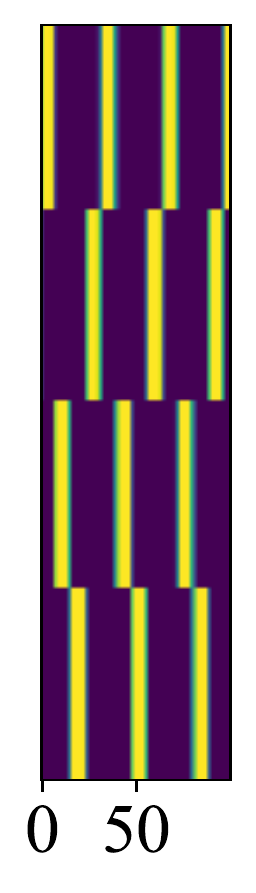}}
\end{subfigure}
\vspace{-0.4cm}
\caption{Wasserstein dictionary learning with $r=4$, $N=100$, and the distance on a torus; dictionaries (left), the translated top right dictionary (middle) code matrices (right)}
\label{fig:game2}
\end{figure}

The additional knowledge of the underlying spaces can be utilized in Wasserstein dictionary learning. To illustrate this, we consider a sequence of figures generated by John Conway's Game of Life, which has a periodic domain. We solve the problems of Wasserstein dictionary learning with two different ground metrics: the usual Euclidian distance in Figure~\ref{fig:game1} and the distance on a torus in Figure~\ref{fig:game2}. It can be seen in Figure~\ref{fig:game2} that all dictionaries are similar up to translations.

The results for Wasserstein dictionary learning on MNIST for different $r$'s are given as follows.
\begin{figure}[H]
\centering
\begin{subfigure}[b]{0.45\columnwidth}
\centerline{\includegraphics[width=\columnwidth]{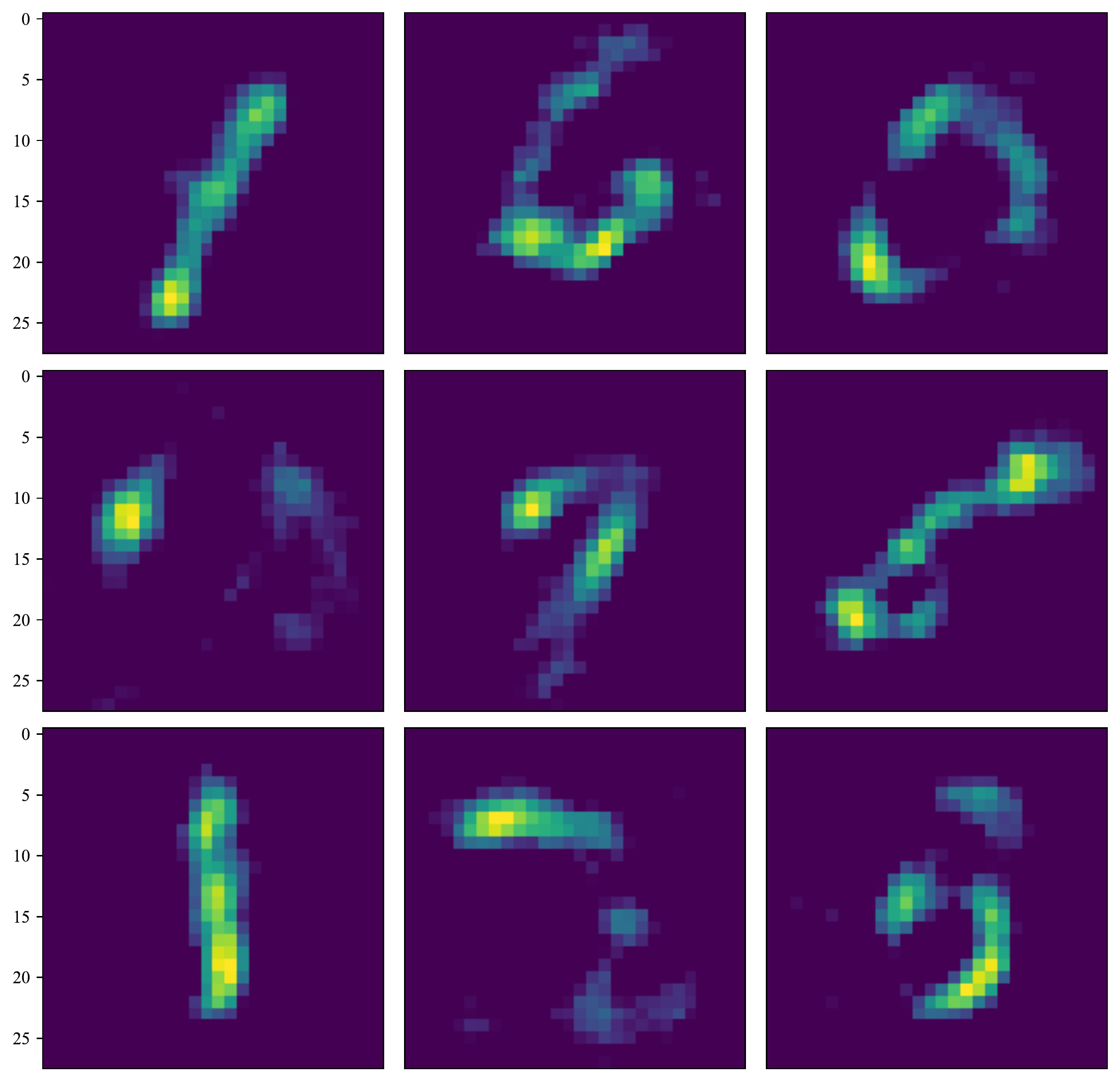}}
\end{subfigure}
\hfill
\begin{subfigure}[b]{0.45\columnwidth}
\centerline{\includegraphics[width=\columnwidth]{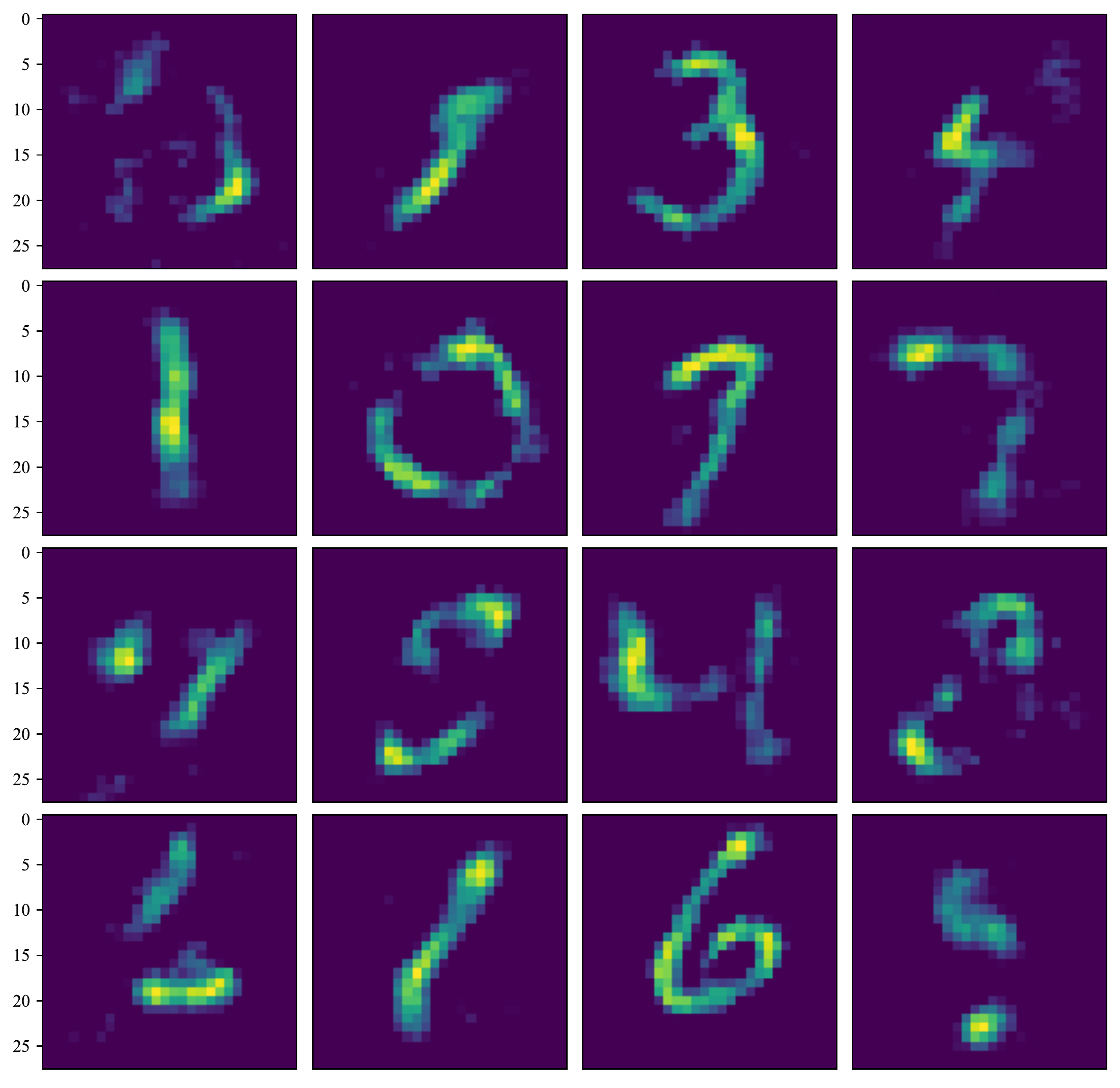}}
\end{subfigure}
\vspace{-0.4cm}
\caption{Wasserstein dictionary learning on MNIST; $r=9$ (left) and $r=16$ (right)}
\label{fig:mnist}
\end{figure}

In Figure \ref{fig:BCD}, we provide a numerical simulation of Algorithm \ref{algorithm:WCPDL_main} for Wasserstein CP-dictionary learning and verify our theoretical convergence results in Theorems \ref{thm:BCD} and \ref{thm:BCD_dWDL}. We observe faster convergence with the presence of proximal regularization with a suitable regularization coefficient. 

\begin{figure}[ht]
		\centering
		\includegraphics[width=1\linewidth]{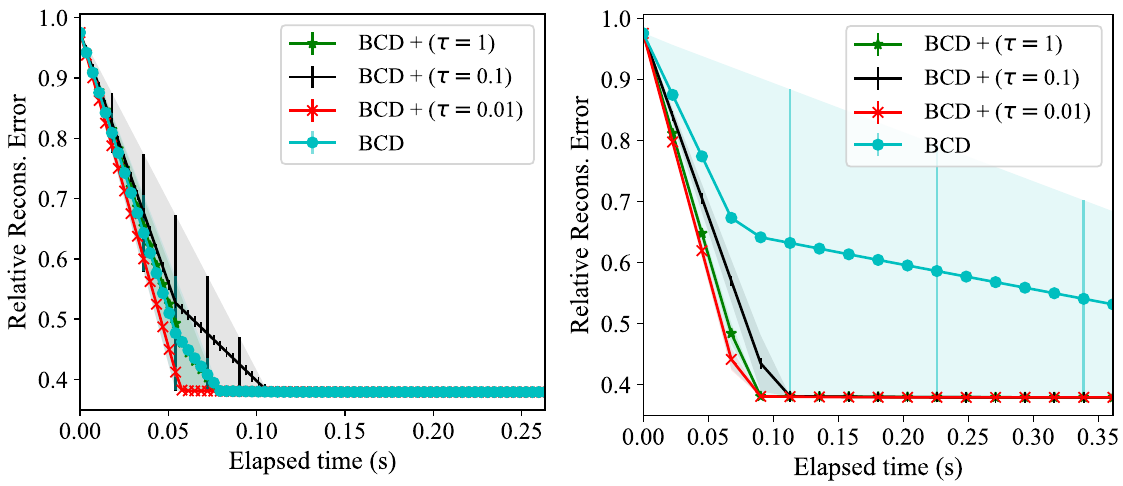} 
        \vspace{-0.7cm}
		\caption{Plot of relative reconstruction error vs. time for Wasserstein CP-dictionary learning using Algorithm \ref{algorithm:WCPDL_main} with various choices of proximal regularization coefficient $\tau\in \{0,0.1,0.01,1\}$. The tensor on the left and right has sizes $(100,100,500)$ and $(100,100,1000)$, respectively. Data tensors are generated by taking the outer product of randomly generated factor matrices of 10 columns plus i.i.d. noise of \textup{Uniform}$(0,10)$. }
		\label{fig:BCD}
	\end{figure}
\section{Conclusion}

We provide a theoretical analysis of the block coordinate descent methods with proximal regularization. The global convergence to the stationary points and the worst-case bound are obtained. We provide Wasserstein CP-dictionary learning as an application of our method.

\section*{Acknowledgements}

DK was supported by the 2023 Research Fund of the University of Seoul. HL was partially supported by the National Science Foundation through grants DMS-2206296 and DMS-2010035.

\bibliography{main_icml}
\bibliographystyle{icml2023}

%%%%%%%%%%%%%%%%%%%%%%%%%%%%%%%%%%%%%%%%%%%%%%%%%%%%%%%%%%%%%%%%%%%%%%%%%%%%%%%
%%%%%%%%%%%%%%%%%%%%%%%%%%%%%%%%%%%%%%%%%%%%%%%%%%%%%%%%%%%%%%%%%%%%%%%%%%%%%%%
% APPENDIX
%%%%%%%%%%%%%%%%%%%%%%%%%%%%%%%%%%%%%%%%%%%%%%%%%%%%%%%%%%%%%%%%%%%%%%%%%%%%%%%
%%%%%%%%%%%%%%%%%%%%%%%%%%%%%%%%%%%%%%%%%%%%%%%%%%%%%%%%%%%%%%%%%%%%%%%%%%%%%%%
\newpage
\appendix
\onecolumn

\section{Block Coordinate Descent with proximal regularization}

\subsection{Proof of Theorem  \ref{thm:BCD}}
    Throughout this section, we let $(\param_{n})_{n\ge 1}$  denote an inexact output of Algorithm \eqref{eq:BCD_factor_update_proximal} and write $\param_{n}=[\theta_{n}^{(1)},\dots,\theta_{n}^{(m)}]$ for each $n\ge 1$. For each $n\ge 1$ and $i=1,\dots,m$, denote 
     \begin{align}
            f_{n}^{(i)}: \theta \mapsto f(\theta_{n}^{(1)},\dots,\theta_{n}^{(i-1)},\theta,\theta_{n-1}^{(i+1)},\dots,\theta_{n-1}^{(m)}),
    \end{align}
    which is $L$-smooth under Assumption~\ref{assumption:A1}. By Lemma \ref{lem:L_smooth_weak_convex}, it is also $L$-weakly convex. From this, it is easy to see that  $g_{n}^{(i)}(\theta)=f_{n}^{(i)}(\theta)+\frac{\tau_{n}^{(i)}}{2}\lVert \theta-\theta_{n-1}^{(i)} \rVert^{2}$ is $(\tau_{n}^{(i)}-L^{(i)})$-strongly convex. Also, denote 
    \begin{align}
        &\tau_{n}^{-}:=\min_{i=1,\dots,m} \tau_{n}^{(i)},\quad  \tau_{n}:=\max_{i=1,\dots,m} \tau_{n}^{(i)} \,\, \textup{for all $n\ge 1$},  \quad L:=\max_{i=1,\dots,m}L^{(i)}. 
    \end{align}
    We will use the notations above as well as this observation throughout this section.

	\begin{prop}[Forward monotonicity]\label{prop:forward_monotonicity_proximal}
		Suppose Assumptions~\ref{assumption:A1}-\ref{assumption:A3}. Then the following hold: 
		\begin{description}[itemsep=0.1cm]
			\item[(i)] $f(\param_{n-1}) - f(\param_{n}) \ge   \frac{\tau_{n}^{-}}{2} \lVert \param_{n-1} - \param_{n} \rVert^{2} - m\Delta_{n}$;
			\vspace{0.1cm}
			\item[(ii)] $\sum_{n=1}^{\infty} \tau_{n}^{-} \rVert \param_{n}-\param_{n-1} \rVert^{2}<\sup_{\param\in \Param} f(\param) + m\sum_{n=1}^{\infty} \Delta_{n} <\infty$.
		\end{description}
	\end{prop}
	
	\begin{proof}
	     Fix $i\in \{1,\dots,m\}$. Let $\theta_{n}^{(i\star)}$ be the exact minimizer of the $(\tau_{n}^{(i)}-L^{(i)})$-strongly convex function $g^{(i)}_{n}(\theta)$ over the convex set $\Theta^{(i)}$. Then $g^{(i)}_{n}(\theta^{(i)}_{n}) \le f^{(i)}_{n}(\theta^{(i)}_{n-1})=g^{(i)}_{n}(\theta^{(i)}_{n-1})$, for $n\ge 1$. Hence we deduce 
		\begin{align}
			f^{(i)}_{n}(\theta^{(i)}_{n-1}) - f^{(i)}_{n}(\theta^{(i)}_{n}) &=   g^{(i)}_{n}(\theta^{(i)}_{n-1}) -  g^{(i)}_{n}(\theta^{(i)}_{n})   +g^{(i)}_{n}(\theta^{(i)}_{n}) -  f^{(i)}_{n}(\theta^{(i)}_{n}) \ge  - \Delta_{n} +  \frac{\tau_{n}^{(i)}}{2} \lVert \theta^{(i)}_{n} - \theta^{(i)}_{n-1}  \rVert^{2}.
		\end{align}
		It follows that 
		\begin{align}
			&f(\param_{n-1}) - f(\param_{n}) \\
			&\qquad = \sum_{i=1}^{n} f([\theta_{n}^{(1)},\dots, \theta_{n}^{(i-1)},\theta_{n-1}^{(i)},\theta_{n-1}^{(i+1)}, \dots, \theta_{n-1}^{(m)}]) - f([\theta_{n}^{(1)},\dots, \theta_{n}^{(i-1)},\theta_{n}^{(i)},\theta_{n-1}^{(i+1)}, \dots, \theta_{n-1}^{(m)}]) \\ 
			&  \qquad = \sum_{i=1}^{n} f_{n}^{(i)}(\theta^{(i)}_{n-1}) -  f_{n}^{(i)}(\theta^{(i)}_{n}) \\
			&\qquad \ge \sum_{i=1}^{m}\left(  \frac{\tau_{n}^{(i)}}{2} \lVert \theta^{(i)}_{n} - \theta^{(i)}_{n-1} \rVert^{2}  -\Delta_{n} \right) = \frac{\tau_{n}^{-}}{2} \lVert \param_{n-1} - \param_{n} \rVert^{2} - m\Delta_{n}.
			\quad  \label{eq:proximal_monotonicity}
		\end{align}
		This shows \textbf{(i)}. 
	
		\iffalse
		Lastly, using \eqref{eq:2nd_order_growth_prox}, we can write 
		\begin{align}
			 \lVert \param_{n-1} - \param_{n} \rVert \le   \lVert \param_{n-1} - \param_{n} \rVert +  \lVert \param_{n} - \param_{n} \rVert \le  \lVert \param_{n-1} - \param_{n} \rVert +  \frac{m}{\tau_{n}} \Delta_{n},
		\end{align}
		from which we deduce 
		\begin{align}
			\tau_{n} \lVert \param_{n-1} - \param_{n} \rVert^{2}  &\ge \tau_{n}\left(  \lVert \param_{n-1} - \param_{n} \rVert  -  \frac{m}{\tau_{n}} \Delta_{n}\right)^{2}  \\
			&\ge \tau_{n} \lVert \param_{n-1} - \param_{n} \rVert^{2} - 2m\Delta_{n} \lVert \param_{n-1} - \param_{n} \rVert + \frac{m^{2}\Delta_{n}^{2}}{\tau_{n}} \\
			&\ge \frac{\tau_{n}}{2} \lVert \param_{n-1} - \param_{n} \rVert^{2},
		\end{align}
		where the last inequliaty uses the hypothesis \eqref{eq:prox_hypothesis_gap}. Then combining with the previous inequality shows \textbf{(i)}. 
		\fi

		Next, to show \textbf{(ii)}, adding up the above inequality, 
		\begin{align}
			\sum_{k=1}^{n} \frac{\tau_{k}^{-}}{2} \lVert \param_{k-1} - \param_{k}\rVert^{2}& \le \left( \sum_{k=1}^{n} f(\param_{k-1}) - f(\param_{k}) \right) + m\sum_{n=1}^{\infty} \Delta_{n} = f(\param_{0}) + m\sum_{n=1}^{\infty} \Delta_{n}    <\infty,
		\end{align}
		where we have used the fact that $\sum_{n=1}^{\infty} \Delta_{n}<\infty$ due to Assumption~\ref{assumption:A3}. 
	\end{proof}

	\begin{prop}[Finite first-order variation]\label{prop:finite_range_short_points_proximal}
		Suppose Assumptions~\ref{assumption:A1}-\ref{assumption:A2}. Also assume $\tau_{n}^{-}\ge 1$ for all $n\ge 1$. Suppose that $\sum_{n=1}^{\infty} \Delta_{n}<\infty$. Then 
		\begin{align*}
			\sum_{n=1}^{\infty}  \left| \left\langle \nabla f(\param_{n+1}),\,  \param_{n} - \param_{n+1}  \right\rangle \right|  < \frac{L+2}{2}  \sup_{\param\in \Param} f(\param)  + 3m\sum_{n=1}^{\infty} \Delta_{n} <\infty.
		\end{align*}
	\end{prop}
	
	\vspace{-0.3cm}
	\begin{proof}
		According to Assumptions~\ref{assumption:A1} and \ref{assumption:A2}, it follows that $\nabla f$ over $\Param$ is Lipschitz with   Lipshitz constant $L$. Hence by Lemma \ref{lem:surrogate_L_gradient}, for all $t\ge 1$,
		\begin{align*}
			\left|f(\param_{n})-f(\param_{n+1}) -  \left\langle \nabla f(\param_{n+1}),\,  \param_{n} - \param_{n+1} \right\rangle \right| \le \frac{L}{2}\lVert \param_{n} - \param_{n+1} \Vert_{F}^{2}.
		\end{align*}
		Using Proposition \ref{prop:forward_monotonicity_proximal}, it follows that
		\begin{align}
		    |f(\param_{n-1})-f(\param_{n})| \le f(\param_{n-1})-f(\param_{n}) + 2m\Delta_{n}.
		\end{align}
		Hence this yields 
		\begin{align}\label{eq:1storder_growth_bd_pf_proximal}
			\left| \left\langle \nabla f(\param_{n+1}),\,  \param_{n} - \param_{n+1}  \right\rangle \right| &\le   \frac{L}{2}\lVert \param_{n} - \param_{n+1} \Vert_{F}^{2} + | f(\param_{n})- f(\param_{n+1}) | \\
			&\le \frac{L}{2}\lVert \param_{n} - \param_{n+1} \Vert_{F}^{2} +  f(\param_{n})- f(\param_{n+1})   + 2m\Delta_{n} 
		\end{align}
		for $n\ge 1$.  Also note that $\sum_{t=1}^{n} f(\param_{t})-f(\param_{t+1}) = f(\param_{1}) - f(\param_{n+1})\le f(\param_{1})$. Hence
			\begin{align*}
				\sum_{n=0}^{\infty} \left| \left\langle \nabla f(\param_{n+1}),\,  \param_{n} - \param_{n+1}  \right\rangle \right| &\le \frac{L}{2} \left( \sum_{n=0}^{\infty}  \lVert \param_{n} - \param_{n+1} \Vert_{F}^{2}\right)   + f(\param_{0}) + 2m\sum_{n=1}^{\infty} \Delta_{n} \\
				&\le \frac{L}{2}  \left( \sum_{n=0}^{\infty}  \tau_{n}^{-} \lVert \param_{n} - \param_{n+1} \Vert_{F}^{2}\right) + f(\param_{0}) + 2m\sum_{n=0}^{\infty} \Delta_{n}  \\
				&\le 2f(\param_{0}) + 3m\sum_{n=1}^{\infty} \Delta_{n}  <\infty,
			\end{align*}
		where we have used Proposition \ref{prop:forward_monotonicity_proximal} \textbf{(ii)}. 
	\end{proof}

    \begin{prop}[Boundedness of iterates]\label{prop:boundedness}
     Under Assumptions~\ref{assumption:A2} and \ref{assumption:A3},  the set $\{ \param_{n}\,:\, n\ge 1 \}$ is bounded.  
    \end{prop}

    \begin{proof}
        Let $T:=m\sum_{k=1}^{\infty} \Delta_{k}$, which is finite by Assumption \ref{assumption:A3}. Recall that by Proposition \ref{prop:forward_monotonicity_proximal}, we have 
        \begin{align}
            \sup_{n\ge 1} f(\param_{n}) \le f(\param_{1}) +T<\infty.
        \end{align}
        Then we can conclude by using Assumption \ref{assumption:A2}. 
    \end{proof}
 
	\begin{prop}[Asymptotic first-order optimality]\label{prop:first_order_optimality_proximal}
		Suppose Assumptions~\ref{assumption:A1}-\ref{assumption:A3}. Fix a sequence $(b_{n})_{n\ge 1}$ with $b_{n}>0$ for $n\ge 1$. Then there exists constants $c_{1},c_{2}>0$ independent of $\param_{0}\in \param$ such that for all $n\ge 1$,
		\begin{align}
			\left\langle \nabla f(\param_{n+1}),\,  \param_{n+1} - \param_{n} \right\rangle &\le  b_{n+1} \inf_{\param \in \param} \left\langle \nabla f(\param_{n}),\, \frac{ \param - \param_{n} }{\lVert \param - \param_{n}\rVert}\right\rangle  + c_{0} b_{n+1} \lVert \param_{n}-\param_{n+1} \rVert \\
            &\qquad + c_{1} \lVert \param_{n+1}-\param_{n} \rVert^{2} + c_{2}(L+\tau_{n+1}) b_{n+1}^{2} + \Delta_{n+1}.
		\end{align}
	\end{prop}

	\begin{proof}
		Fix arbitrary $\param=[\theta^{(1)},\dots,\theta^{(m)}]\in \Param$ such that $\lVert \param - \param_{n} \rVert \le b_{n+1}$. By  convexity of  $\Theta^{(i)}$, $\theta_{n}^{(i)}+a(\theta^{(i)}-\theta_{n}^{(i)})\in \Theta^{(i)}$ for all $a\in [0,1]$.  Let $\theta^{(i\star)}_{n+1}$ denote the exact minimizer of $g_{n+1}^{(i)}$ over $\Theta^{(i)}$. Then we have 
		\begin{align}
			f_{n+1}^{(i)}( \theta^{(i)}_{n+1} )   +  \frac{\tau_{n+1}^{(i)}}{2} \lVert \theta^{(i)}_{n+1} - \theta^{(i)}_{n}  \rVert^{2}  -\Delta_{n+1} &\le 
			f_{n+1}^{(i)}( \theta^{(i\star)}_{n+1} )   +  \tau_{n+1}^{(i)} \lVert \theta^{(i\star)}_{n+1} - \theta^{(i)}_{n}  \rVert^{2}  \\
			& \le f_{n+1}^{(i)}\left( \theta_{n}^{(i)}+ a ( \theta^{(i)} - \theta^{(i)}_{n} )  \right)  +  \frac{\tau_{n+1}^{(i)} a^{2} }{2} \lVert \theta^{(i)}-\theta_{n}^{(i)} \rVert^{2}.
		\end{align}
	Recall that each $f_{n+1}^{(i)}$ is $L^{(i)}$-smooth by Assumption~\ref{assumption:A1}. Hence by subtracting $f_{n+1}^{(i)}(\theta_{n}^{(i)})$ from both sides and using  Lemma \ref{lem:surrogate_L_gradient}, we get 
		\begin{align}
			\left\langle \nabla f_{n+1}^{(i)}(\theta_{n}^{(i)}),\,  \theta_{n+1}^{(i)}-\theta_{n}^{(i)}  \right\rangle &\le a \left\langle \nabla f_{n+1}^{(i)}(\theta_{n}^{(i)}),\,  \theta^{(i)}-\theta_{n}^{(i)}\right\rangle   \\
			&\qquad + \frac{L^{(i)}}{2}\lVert \theta_{n+1}^{(i)} - \theta_{n}^{(i)} \rVert^{2} + \frac{L^{(i)}}{2}\lVert \theta^{(i)} - \theta_{n}^{(i)} \rVert^{2} + \frac{\tau_{n+1}^{(i)} a^{2}}{2}\lVert \theta^{(i)}-\theta_{n}^{(i)} \rVert^{2} +  \Delta_{n+1}.
		\end{align}
		Adding up these inequalities for $i=1,\dots,m$ , 
		\begin{align}
			\left\langle \left[ \nabla f_{n+1}^{(1)}(\theta_{n}^{(1)}),\dots,\nabla f_{n+1}^{(m)}(\theta_{n}^{(m)}) \right],\, \param_{n+1} - \param_{n} \right\rangle &\le a  \left\langle \left[ \nabla f_{n+1}^{(1)}(\theta_{n}^{(1)}),\dots,\nabla f_{n+1}^{(m)}(\theta_{n}^{(m)}) \right],\,  \param - \param_{n} \right\rangle \\
			&\hspace{-1cm} + \frac{L}{2}\lVert \param_{n+1} - \param_{n}  \rVert^{2} + \frac{(L+\tau_{n+1}a^{2})}{2}\lVert \param - \param_{n}\rVert^{2}  +  \Delta_{n+1}.
		\end{align}

        \iffalse
        \begin{align}
        \nabla f(\param_{n}) = [ \nabla_{1} f(\theta_{n}^{(1)},\dots, \theta_{n}^{(m)}),\nabla_{2} f(\theta_{n}^{(1)},\dots, \theta_{n}^{(m)}),\dots, \nabla_{m} f(\theta_{n}^{(1)},\dots, \theta_{n}^{(m)})  ] 
        \end{align}

        \begin{align}
        \left[ \nabla f_{n+1}^{(1)}(\theta_{n}^{(1)}),\dots,\nabla f_{n+1}^{(m)}(\theta_{n}^{(m)}) \right] = \left[ \nabla_{1} f(\theta_{n}^{(1)}, \theta_{n}^{(2)},\dots,\theta_{n}^{(m)} ),\dots, \nabla_{m} f(\theta_{n+1}^{(1)}, \dots, \theta_{n+1}^{(m-1)},\theta_{n}^{(m)} )\right]
        \end{align}
        \fi

		\noindent Since for each $i=1,\dots,m$ $\nabla f$ is $L^{(i)}$-Lipschits in the $i$th block coordinate,  we have 
        \begin{align}
    \lVert   \nabla_{i} f(\theta_{n}^{(1)},\dots, \theta_{n}^{(m)}) - \nabla f_{n+1}^{(i)}(\theta_{n}^{(i)}) \rVert \le L^{(i)} \lVert \param_{n}-\param_{n+1} \rVert.
        \end{align}
    Hence there exists constants $c_{1},c_{2}>0$ independent of $\param_{0}\in \Param$, such that 
		\begin{align}\label{eq:optimality1_proximal}
			\hspace{-0.3cm}	\left\langle \nabla f(\param_{n+1}),\,  \param_{n+1} - \param_{n} \right\rangle &\le a  \left\langle \nabla f(\param_{n}),\, \param-\param_{n} \right\rangle  + a m L \lVert \param_{n}-\param_{n+1} \rVert\cdot  \lVert \param -\param_{n} \rVert \\
			&\qquad + c_{1} \lVert \param_{n+1}-\param_{n} \rVert^{2} + c_{2}(L+\tau_{n+1} a^{2})\lVert \param - \param_{n}\rVert^{2} + \Delta_{n+1}.
		\end{align}
		 The above inequality holds for all $a\in [0,1]$. 
   
    Viewing the right hand side as a quadratic function in $a$, the only possibly negative term is the linear term $a  \left\langle \nabla f(\param_{n}),\, \param-\param_{n} \right\rangle $, whose absolute value is bounded above by $a \lVert \nabla f(\param_{n}) \rVert \lVert \param-\param_{n} \rVert $. By Proposition \ref{prop:boundedness} and Assumption~\ref{assumption:A3}, $\lVert \nabla f (\param_{n}) \rVert$ is uniformly bounded, so this is bounded above by $a c_{3} \lVert \param-\param_{n} \rVert$ for some constant $c_{3}>0$. Hence we may choose $c_{2}>0$ large enough so that the right hand side above is non-increasing in $a$. Thus the inequality above holds for all $a\ge 0$. In particular, we can choose $a=b_{n+1}/\lVert \param-\param_{n} \rVert$. This and using $\lVert \param-\param_{n} \rVert \le b_{n+1}$ yield
		\begin{align}\label{eq:optimality1_proximal2}
			\hspace{-0.3cm}	\left\langle \nabla f(\param_{n+1}),\,  \param_{n+1} - \param_{n} \right\rangle &\le b_{n+1}  \left\langle \nabla f(\param_{n}),\, \frac{\param-\param_{n}}{\lVert \param-\param_{n} \rVert} \right\rangle  + c_{0} \lVert \param_{n}-\param_{n+1} \rVert  b_{n+1} \\
			&\qquad + c_{1} \lVert \param_{n+1}-\param_{n} \rVert^{2} + c_{2}(L+\tau_{n+1}) b_{n+1}^{2} + \Delta_{n+1},
		\end{align}
		where we wrote $c_{0}:=mL$. 
	
  We have shown that the above holds  for all $\param\in \Param$ such that $\lVert \param - \param_{n} \rVert \le  b_{n+1}$. It remains to argue  that \eqref{eq:optimality1_proximal2} also holds for all $\param\in\Param$ with $\lVert \param-\param_{n} \rVert\ge b_{n+1}$. Indeed, for such $\param$, let $\param'$ be the point in the secant line between $\param$ and $\param_{n}$ such that $\lVert \param'-\param_{n} \rVert \le b_{n+1}$. Then $\param'\in \Param$ and \eqref{eq:optimality1_proximal2} holds for $\param$ replaced with $\param'$. However, the right hand side is unchanged when replacing $\param$ with any point on the line passing through $\param$ and $\param_{n}$. Thus \eqref{eq:optimality1_proximal2} holds for all $\param\in \Param$. This shows the assertion. 
	\end{proof}

	%Recall that during the update $\param_{n-1}\mapsto \param_{n}$ each block coordinate of $\param_{n-1}$ changes by at most $r_{n}$ in Frobenius norm. For each $n\ge 1$, we say $\param_{n}$ is a \textit{long point} if none of the block coordinates of $\param_{n-1}$ change by $r_{n}$ in Frobenius norm and a \textit{short point} otherwise. Observe that if $\param_{n}$ is a long point, then imposing the search radius restriction in \eqref{eq:BCD_DR_block} has no effect and $\param_{n}$ is obtained from $\param_{n-1}$ by a single cycle of block coordinate descent on $f$ over $\Param$. 

	\begin{prop}[Optimality gap for iterates]\label{prop:iterate_opt_gap}
		For each $n\ge 1$ and $i\in \{1,\dots,m\}$, let $\theta_{n}^{(i\star)}$ be the exact minimizer of the $(\tau_{n}^{(i)}-L^{(i)})$-strongly convex function $\theta \mapsto g^{(i)}_{n}(\theta)$ in \eqref{eq:BCD_factor_update_proximal} over the convex set $\Theta^{(i)}$. Then 
		\begin{align}\label{eq:iterate_optimality_gap_prox}
			\frac{\tau_{n}^{(i)}-L^{(i)}}{2} \lVert \theta^{(i\star)}_{n} - \theta^{(i)}_{n} \rVert^{2} \le  \Delta_{n}.
		\end{align}
	\end{prop}
	
	\begin{proof}
		The assertion follows from 
		\begin{align}\label{eq:2nd_order_growth_prox}
			\frac{\tau_{n}^{(i)}-L^{(i)}}{2} \lVert \theta^{(i\star)}_{n} - \theta^{(i)}_{n} \rVert^{2}  \le g^{(i)}_{n}(\theta^{(i)}_{n}) - g^{(i)}_{n}(\theta^{(i\star)}_{n} ) \le \Delta_{n}
		\end{align}
		for $n\ge 1$. Indeed, the first inequality follows from the second-order growth property (see Lemma \ref{lem:second_order_growh_univariate}) since  $g_{n}^{(n)}$ is $(\tau_{n}^{(i)}-L^{(i)})$-strongly convex minimized at $\theta^{(i)}_{n}$, and the second inequality follows from the definition of optimality gap $\Delta_{n}$ in \eqref{eq:def_sub_optimality_gap}.
	\end{proof}

	We are now ready to give a proof of Theorem \ref{thm:BCD}.

	\vspace{0.1cm}
	\begin{proof}[\textbf{Proof of Theorem \ref{thm:BCD}}]
		Suppose Assumptions~\ref{assumption:A1}-\ref{assumption:A3} and $\tau_{n}^{(i)}>L^{(i)}+\delta$ for $n\ge 1$ for some $\delta>0$. Also assume $\tau_{n}^{(i)}=O(1)$. We first show \textbf{(i)}. Fix a convergent subsequence $(\param_{n_{k}})_{k\ge 1}$ of $(\param_{n})_{n\ge 1}$. We wish to show that  $\param_{\infty}=\lim_{k\rightarrow \infty} \param_{n_{k}}$ is a stationary point of $f$ over $\Param$. To this end, for each $i\in \{1,\cdots, m\}$, let $\theta^{(i\star)}_{n}$ denote the exact minimizer of the $(\tau_{n}^{(i)}-L^{(i)})$-strongly convex function $g^{(i)}_{n}$ defined in \eqref{eq:BCD_factor_update_proximal}. By using the first-order optimality of $\theta^{(i\star)}_{n}$, we have  
	 	\begin{align}
	 		\left\langle \nabla g_{n}^{(i)} (\theta^{(i\star)}_{n})  ,\, \theta - \theta^{(i\star)}_{n}  \right\rangle = \left\langle \nabla f_{n}^{(i)} (\theta^{(i\star)}_{n}) + \tau_{n}^{(i)} (\theta^{(i\star)}_{n}- \theta^{(i)}_{n-1}) ,\, \theta - \theta^{(i\star)}_{n}  \right\rangle \ge 0 \qquad \forall \theta\in \Theta^{(i)}.
	 	\end{align}
        %\commHL{Need to compare the above inner product with tangent vector $\theta-\theta_{n}^{(i)}$}
        Let $T:=m\sum_{k=1}^{\infty} \Delta_{k}$, which is finite by Assumption \ref{assumption:A3}. Recall that by Proposition \ref{prop:forward_monotonicity_proximal}, we have 
        \begin{align}
            \sup_{n\ge 1} f(\param_{n}) \le f(\param_{1}) +T<\infty.
        \end{align}
        Let $K:=\{\param \,:\, f(\param) \le f(\param_{1})+T \} $ and let $K(T):=\{\param\,:\, \textup{$\exists \param'\in K$ s.t. $\lVert \param-\param' \rVert \le T$} \}$ denote the $T$-neighborhood of $K$. By Assumption \ref{assumption:A2}, $K$ is compact, so $K(T)$ is also compact. Since $f$ is $L^{(i)}$-smooth in its $i$th blook coordinate, $\lVert \nabla g_{n}^{(i)} \rVert$ is uniformly bounded over $\param\in K(T)$ by some constant, say, $L_{K}>0$. Now observe that 
        \begin{align}
	 		\left\langle \nabla g_{n}^{(i)} (\theta^{(i\star)}_{n})  ,\, \theta - \theta^{(i)}_{n}  \right\rangle &\ge \left\langle \nabla g_{n}^{(i)} (\theta^{(i\star)}_{n})  ,\, \theta - \theta^{(i\star)}_{n}  \right\rangle  - \left| \left\langle \nabla g_{n}^{(i)} (\theta^{(i\star)}_{n})  ,\, \theta_{n}^{(i)} - \theta^{(i\star)}_{n}  \right\rangle  \right| \\
    &\ge - \lVert \nabla g_{n}^{(i)} (\theta^{(i\star)}_{n}) \rVert \, \lVert \theta_{n}^{(i)} - \theta_{n}^{(i\star)} \rVert  \\
    &\ge - L_{K} \Delta_{n}.
	 	\end{align}
    %Since $\Delta_{n}=o(1)$ by Assumption \ref{assumption:A3}, it follows that 

	 	Next, using $L^{(i)}$-Lipschitzness of $\nabla f$ in the $i$th block coordinate and Proposition \ref{prop:iterate_opt_gap}, we have 
	 	\begin{align}
	 		&\left| \left\langle \nabla g_{n}^{(i)} (\theta^{(i\star)}_{n})  ,\, \theta - \theta^{(i)}_{n}  \right\rangle  - 	\left\langle \nabla g_{n}^{(i)} (\theta^{(i)}_{n})  ,\, \theta - \theta^{(i)}_{n}  \right\rangle \right| \\
	 		&\qquad \le \left| \left\langle \nabla f_{n}^{(i)} (\theta^{(i\star)}_{n})  - \nabla f_{n}^{(i)} (\theta^{(i)}_{n})  + \tau_{n}^{(i)} (\theta^{(i\star)}_{n}- \theta^{(i)}_{n}) ,\, \theta - \theta^{(i)}_{n}  \right\rangle  \right|  \\
	 		&\qquad \le  \left( \lVert  \nabla f_{n}^{(i)} (\theta^{(i\star)}_{n})  -  \nabla f_{n}^{(i)} (\theta^{(i)}_{n})   \rVert  + \tau_{n}^{(i)} \lVert \theta^{(i\star)}_{n} - \theta^{(i)}_{n} \rVert\right) \lVert \theta - \theta_{n}^{(i)} \rVert \\
	 		&\qquad \le (L^{(i)}+\tau_{n}^{(i)}) \lVert \theta - \theta_{n}^{(i)} \rVert \, \lVert \theta^{(i\star)}_{n} - \theta_{n}^{(i)}\rVert  \\
	 		&\qquad \le (L^{(i)}+\tau_{n}^{(i)}) \lVert \theta - \theta_{n}^{(i)} \rVert \sqrt{\frac{2\Delta_{n}}{\tau_{n}^{(i)}-L^{(i)}}} = \lVert \theta - \theta_{n}^{(i)} \rVert \sqrt{\frac{8\tau_{n}^{(i)}\Delta_{n} }{1-L^{(i)}/\tau_{n}^{(i)}}}, 
	 	\end{align}
	 	where for the last equality we have used that $\tau_{n}^{(i)}>L^{(i)}$ for $n\ge 1$. From Assumption~\ref{assumption:A3}, we can deduce $\Delta_{n}=o(1)$. Using the hypotheses $\tau_{n}^{(i)}>L^{(i)}+\delta$ for $n\ge 1$ for some $\delta>0$ (see Algorithm \eqref{eq:BCD_factor_update_proximal}),  and $\tau_{n}^{(i)}=O(1)$, we see that the term inside the square root in the last expression is $o(1)$. Furthermore, $\lVert \theta - \theta_{n_{k}}^{(i)} \rVert$ is uniformly bounded in $k$ since $\theta_{n_{k}}^{(i)}$ converges as $k\rightarrow \infty$. Hence 
	 	\begin{align}
	 	\liminf_{k\rightarrow \infty} \, \left\langle \nabla g_{n_{k}}^{(i)} (\theta^{(i)}_{n_{k}})  ,\, \theta - \theta^{(i)}_{n_{k}}  \right\rangle \ge 0 \qquad \forall \theta\in \Theta^{(i)}.
	 	\end{align}
	 	
	 	Note that by Proposition \ref{prop:forward_monotonicity_proximal} \textbf{(ii)} and $\tau_{n}^{(i)}=O(1)$, we get $\tau_{n}^{(i)} \lVert \param_{n}-\param_{n-1} \rVert=o(1)$. So if we write  $\param_{\infty}=[\theta_{\infty}^{(1)},\dots,\theta_{\infty}^{(m)}]$, For each $\theta\in \Theta^{(i)}$, by the hypothesis, we get 
	 	\begin{align}
	 		&\lim_{k\rightarrow\infty}\,  \left| \left\langle \nabla f_{n_{k}}^{(i)} (\theta^{(i)}_{n_{k}}) + 2\tau_{n_{k}}^{(i)} (\theta^{(i)}_{n_{k}}- \theta^{(i)}_{n_{k}-1}) ,\, \theta - \theta^{(i)}_{n_{k}}  \right\rangle  - \left\langle \nabla f_{n_{k}}^{(i)} (\theta^{(i)}_{n_{k}}) ,\, \theta - \theta^{(i)}_{n_{k}}  \right\rangle  \right| \\
	 		&\qquad \le  \lim_{k\rightarrow\infty}\,  2\tau_{n_{k}}^{(i)} \lVert \theta^{(i)}_{n_{k}}- \theta^{(i)}_{n_{k}-1} \rVert \, \lVert \theta - \theta^{(i)}_{n_{k}-1} \rVert \\
	 		&\qquad =  2\lVert \theta - \theta^{(i)}_{\infty} \rVert  \lim_{k\rightarrow\infty}\,  \tau_{n_{k}}^{(i)} \lVert \theta^{(i)}_{n_{k}}- \theta^{(i)}_{n_{k}-1} \rVert = 0.
	 	\end{align}
	 	It follows that, for each $\theta\in \Theta^{(i)}$, using the continuity of $\nabla f$ in Assumption~\ref{assumption:A2}, 
	 	\begin{align}
	 	\left\langle \nabla_{i} f( \theta^{(1)}_{\infty},\dots,\theta^{(i-1)}_{\infty}, \theta^{(i)}_{\infty}, \theta^{(i+1)}_{\infty},\dots, \theta^{(m)}_{\infty} ) ,\, \theta - \theta_{\infty}^{(i)}   \right\rangle  =\lim_{k\rightarrow\infty} 	\left\langle \nabla f_{n_{k}}^{(i)} (\theta^{(i)}_{n_{k}}) ,\, \theta - \theta^{(i)}_{n_{k}}  \right\rangle  \ge 0.
	 	\end{align}
 		This holds for all $i=1,\dots,m$. Therefore we verify $	\left\langle \nabla f (\param_{\infty} ) ,\, \param - \param_{\infty}  \right\rangle   \ge 0 $ for all $\param\in \Param$,  which means that $\param_{\infty}$ is a stationary point of $f$ over $\Param$, as desired. This shows \textbf{(i)}.

	    Next, we show \textbf{(ii)}. Let $b_{n}$ be any square-summable sequence of positive numbers. By Cauchy-Schwarz inequality,
        \begin{align}
        \sum_{k=1}^{n} b_{k} \lVert \param_{k}-\param_{k+1} \rVert  \le \left( \sum_{k=1}^{n} b_{k}^{2}  \right)^{1/2} \left( \sum_{k=1}^{n} \lVert \param_{k}-\param_{k+1} \rVert^{2} \right)^{1/2}. 
        \end{align}
        Then by Proposition \ref{prop:forward_monotonicity_proximal}, the right hand side is uniformly bounded in $n\ge 1$, so we see that the left hand side is also uniformly bounded in $n$. Hence using Propositions \ref{prop:forward_monotonicity_proximal} and \ref{prop:finite_range_short_points_proximal}, 
		\begin{align}\label{eq:convergence_rate_ineq1}
		\sum_{n=1}^{\infty}	b_{n+1} \left[ -\inf_{\param\in \Param} \left\langle \nabla f(\param_{n}),\, \frac{\param - \param_{n} }{\lVert \param - \param_{n}\rVert}\right\rangle\right] \le C\left( \sup_{\param\in \Param} f(\param) + \sum_{n=1}^{\infty} \Delta_{n}(\param_{0})  + \sum_{n=1}^{\infty} b_{n}^{2} + \sum_{n=1}^{\infty} b_{n}\lVert \param_{n}-\param_{n+1} \rVert \right)
		\end{align}
		for some constant $C>0$ independent of $\param_{0}$,  and the right hand side is finite.  Thus by taking $b_{n}=1/( \sqrt{n} \log n  )$, using Lemma \ref{lem:positive_convergence_lemma}, we deduce 
		\begin{align}\label{eq:thm1_rate_bd}
			\min_{1\le k \le n}  \,\, \left[ -\inf_{\param\in \Param} \left\langle \nabla f(\param_{k}),\, \frac{\param - \param_{k} }{\lVert \param - \param_{k}\rVert}\right\rangle  \right]  \le \frac{M+c \sum_{n=1}^{\infty} \Delta_{n}(\param_{0})}{\sqrt{n}/\log n}
		\end{align}
		for some constants $M,c>0$ independent of $\param_{0}$. This shows \textbf{(ii)}. 
		
		Lastly, we show \textbf{(iii)}. Assume $\sup_{\param_{0}\in \Param} \sum_{n=1}^{\infty} \Delta_{n}(\param_{0})<\infty$. Then the above implies that for some constant $M'>0$ independent of $\param_{0}$, \begin{align}\label{eq:thm1_rate_bd2}
			\min_{1\le k \le n}  \sup_{\param_{0}\in \Param}\,\, \left[ -\inf_{\param\in \Param} \left\langle \nabla f(\param_{k}),\, \frac{\param - \param_{k} }{\lVert \param - \param_{k}\rVert}\right\rangle  \right]^{2}  \le \frac{M'(\log n)^{2} }{n}. 
		\end{align}
		Then one can conclude \textbf{(iii)} by using the fact that $n\ge 2\eps^{-1} (\log \eps^{-1})^{2}$ implies $(\log n)^{2}/n \le \eps$ for all sufficiently small $\eps>0$. This completes the proof. 
		\end{proof}

	%We postpone the proof to the appendix (see Lemma~\ref{lem:fdual}).
	
	%As $\lambda$ and $\lambda_0$ are in 
	%This implies that

	%\commHL{add discussion on per-iteration complexity}

\section{Proof of Theorem~\ref{lem:per_iter_correctness}}
\label{ap:per_iter_correctness}

%\newpage

In this section, we establish Theorem~\ref{lem:per_iter_correctness}, the per-iteration correctness of Algorithm~\ref{algorithm:dWDL_main}. This directly follows from Propositions~\ref{prop:per_iter_correctness_Lambda} and \ref{prop:per_iter_correctness_D} below.

\begin{prop}\label{prop:per_iter_correctness_Lambda}
For given $(\mathcal{D}_{n-1}, \Lambda_{n-1}) \in \Sigma^{r}_{I_{1}\times \cdots \times I_{d}}\times \Sigma^{N}_{r}$ and $\tau_n > 0$, let $\Lambda_{n}\in \Sigma^{N}_{r}$ be a solution of \eqref{eq:primal}. Suppose each fiber of $\mathcal{D}_{n-1}$ along the last mode is not identically zero. Then, $\Lambda_n$ is uniquely determined by 
\begin{align}
\label{eq:lam}
\Lambda_n = \left(\Lambda_{n-1} + \frac{\mathcal{D}_{n-1}\times_{\leq d} G^{\circ}_n}{\tau_n} - J^\circ \otimes c^{\circ}_n\right)_+.
\end{align}
Here, $G^{\circ}_n \in \mathbb{R}^N$ is defined as the unique solution of the dual problem \eqref{eq:dual},  $c^{\circ}_n \in \mathbb{R}^{N\times 1}$ is chosen to satisfy $\Lambda_n \in \Sigma^{N}_{r}$ and all entries of $J^\circ \in \R^{r \times 1}$ are one.
\end{prop}

As shown later in the proof, the assumption on $\mathcal{D}_{n-1}$ in the above proposition is required to ensure the above derivation. It is worth pointing out that it can be easily achieved in the algorithm by adding small noise, if necessary.

\begin{prop}
\label{prop:per_iter_correctness_D}
For given $(\mathcal{D}_{n-1}, \Lambda_{n}) \in \Sigma^{r}_{I_{1}\times \cdots \times I_{d}}\times \Sigma^{N}_{r}$, let $\mathcal{D}_{n}\in \Sigma^{r}_{I_{1}\times \cdots \times I_{d}}$ be a solution of \eqref{alg:dWDL_high-level_D}. Suppose each fiber of $\Lambda_{n}\in \Sigma^{N}_{r}$ along the $2$nd mode is not identically zero. Then 
$\mathcal{D}_n$ is uniquely determined by 
\begin{align}
\mathcal{D}_{n} = \left(\mathcal{D}_{n-1} + \frac{G^{\dagger}_n \times_{d+1} \Lambda^{T}_n}{\tau_n}-J^{\dagger} \otimes c^{\dagger}_n\right)_+.
\end{align}
Here, $G^{\dagger}_n$ is defined as the unique solution of the dual problem \eqref{eq:dual2}, $c^{\dagger}_n \in \mathbb{R}^{r\times 1}$ is chosen to satisfy $\mathcal{D}_{n} \in \Sigma^{r}_{I_{1}\times \cdots \times I_{d}}$ and all entries of $J^\dagger \in \R^{I_{1}I_{2} \cdots I_{d} \times 1}$ are one. 
\end{prop}

The following definitions are taken from \cite{bauschke2011convex}.

\begin{definition}
\label{def:list}
\cite{bauschke2011convex}(Definitions 9.12, 19.10, 19.15 \& 19.22)
\begin{itemize}
    \item
    For a nonempty closed convex cone $K \subset \mathcal{K}$, we say that $R:\mathcal{H} \rightarrow \mathcal{K}$ is convex with respect to $K$ if
    $$
    R(\alpha x + (1-\alpha)y) - \alpha Rx - (1-\alpha) Ry \in K
    $$
    for all $x,y \in \mathcal{H}$ and $\alpha \in (0,1)$.
    \item
    The set of proper lower semicontinuous convex functions from $\mathcal{H}$ to $(-\infty, +\infty]$ is denoted by $\Gamma_0(\mathcal{H})$.
    \item
    The Lagrangian of $\mathcal{J}: \mathcal{H} \times \mathcal{K} \rightarrow (-\infty, +\infty]$ is a function given as 
    \begin{align}
    \mathcal{L}: \mathcal{H} \times \mathcal{K} \rightarrow [-\infty, +\infty]: (x,v) \mapsto \inf_{y \in \mathcal{K}} \left( \mathcal{J}(x,y) + \langle y, v \rangle\right).
    \end{align}
    Moreover, $(x,v) \in \mathcal{H} \times \mathcal{K}$ is a saddle point of $\mathcal{L}$ if
    $$\mathcal{L}(x,v) = \sup \mathcal{L} (x, \mathcal{K}) = \inf \mathcal{L} (\mathcal{H}, v).$$
    \item
    The primal problem and the dual problem of $\mathcal{J}: \mathcal{H} \times \mathcal{K} \rightarrow (-\infty, +\infty]$ are respectively given as 
    \begin{align}
        \min_{x \in \mathcal{H}} \mathcal{J}(x,0), \ \ \ \hbox{ and } \ \ \  \min_{v \in \mathcal{K}} \mathcal{J}^*(0,v).
    \end{align}
\end{itemize}
\end{definition}

We first observe that the primal problem of
\begin{align}
\label{eq:gi}
\mathcal{J}: \mathcal{H} \times \mathcal{K} \rightarrow (-\infty, + \infty]: (x,y) \mapsto
\begin{cases}
f(Rx - y) + h(x), &\hbox{ if } Rx \in y+K,\\
+\infty, &\hbox{ if } Rx \notin y+K,
\end{cases}
\end{align}
is the minimization problem \eqref{eq:gprimal}. Its dual problem, the Lagrangian of $\mathcal{J}$, and the saddle point are given in the following lemma.

\begin{lemma}[Characterization of saddle point for general  coding problem]
\label{lem:saddle}
Let $f \in \Gamma_0(\mathcal{K})$, $h \in \Gamma_0(\mathcal{H})$, and $K$ be a nonempty closed convex cone in $\mathcal{K}$. Let $R: \mathcal{H} \rightarrow \mathcal{K}$ be continuous, convex with respect to $K$ such that $K \cap R(\textnormal{dom} h) \neq \emptyset$. For $\mathcal{J}$ given in \eqref{eq:gi}, the following hold:
\begin{enumerate}
    \item
    The dual problem of $\mathcal{J}$ is given as
    \begin{align}
    \label{eq:gdual}
        \min_{v \in \mathcal{K}} f^*(-v;K) + h^*(R^*v)
    \end{align}
    where $f^*(\cdot;K) = \sup_{z \in K} \langle z, \cdot \rangle - f(z)$.
    \item
    The Lagrangian $\mathcal{L}: \mathcal{H} \times \mathcal{K} \to [-\infty, +\infty]$ is given as
    \begin{align}
        \mathcal{L}(x,v) = 
        \begin{cases}
            -\infty &\hbox{ if } x \in \textnormal{dom} h \hbox{ and } v \notin \textnormal{dom} f^*(\cdot;K);\\
            - f^*(v;K) + h(x) + \langle Rx, v \rangle &\hbox{ if } x \in \textnormal{dom} h \hbox{ and } v \in \textnormal{dom} f^*(\cdot;K);\\
            +\infty &\hbox{ if } x \notin \textnormal{dom} h.
        \end{cases}
    \end{align}
    \item
    Suppose that the optimal values $\mu$ and $\mu^*$ of the primal problem and the dual problem satisfy the strong duality $\mu = - \mu^*$. Then,  $(x^\circ, -v^\circ) \in \mathcal{H} \times \mathcal{K}$ is a saddle point of $\mathcal{L}$ if and only if
    \begin{align*}
        x^\circ \in \textnormal{dom} h, \ \ R x^\circ \in K, \ \ -v^\circ \in \textnormal{dom} f^*(\cdot;K),\\
        R^*v^\circ \in \partial h(x^\circ) \hbox{ and } -v^\circ \in \partial f (R x^\circ).
    \end{align*}
\end{enumerate}
\end{lemma}

\begin{proof}

(1): For any $v \in \mathcal{K}$, it holds that
\begin{align*}
    \mathcal{J}^*(0, v) &= \sup_{(x,y) \in \mathcal{H} \times \mathcal{K}
    } \langle y, v \rangle - \mathcal{J}(x,y),\\
    &= \sup_{(x,y) \in \mathcal{H} \times \mathcal{K} \textup{ s.t. } Rx - y \in K} \langle y, v \rangle - h(x) - f(Rx - y),\\
    &= \sup_{(x,z) \in \mathcal{H} \times K } \langle x, R^* v \rangle - h(x) + \langle z, -v \rangle - f(z),\\
    &= h^*(R^*v) + f^*(-v;K).
\end{align*}
From the definition of the dual problem, we conclude.

(2): 
If $x \notin \textnormal{dom} h$, then $h(x) = \infty$. As $f \in \Gamma_0(\mathcal{K})$, we have $\mathcal{J}(x,v) = \infty$ and thus the Lagrangian $\mathcal{L}(x,v) = \infty$. For $x \in \textnormal{dom} h \hbox{ and } v \in \textnormal{dom} f^*(\cdot;K)$, we have
\begin{align*}
    \mathcal{L}(x,v) &= h(x) + \inf_{y \in \mathcal{K} \textup{ s.t. } Rx - y \in K} f(Rx - y) + \langle y, v \rangle,\\
    &= h(x) + \langle Rx, v \rangle + \inf_{z \in K} f(z) - \langle z, v \rangle,\\
    &= h(x) + \langle Rx, v \rangle - \sup_{z \in K} \langle z, v \rangle- f(z),\\
    &= h(x) + \langle Rx, v \rangle - f^*(v;K).
\end{align*}
If $x \in \textnormal{dom} h \hbox{ and } v \notin \textnormal{dom} f^*(\cdot;K)$, the above relation yields that $\mathcal{L}(x,v) = -\infty$.

(3): From $f \in \Gamma_0(\mathcal{K})$, $h \in \Gamma_0(\mathcal{H})$, and the convexity of $R$ with respect to $K$, we have that $\mathcal{J} \in \Gamma_0(\mathcal{H} \times \mathcal{K})$. Applying Corollary 19.17 in \cite{bauschke2011convex}, we obtain that $(x^\circ, -v^\circ)$ is a saddle point of $\mathcal{L}$ if and only if $x^\circ$ is a solution of the primal problem \eqref{eq:gprimal} and $v^\circ$ is a solution of the dual problem \eqref{eq:gdual}.

As $\mu = - \mu^*$, the equivalence in Corollary 19.1 from \cite{bauschke2011convex} concludes our claim.
\end{proof}

Now we are ready to prove Proposition \ref{prop:per_iter_correctness_Lambda}.
	
\begin{proof}[\textbf{Proof of Proposition} \ref{prop:per_iter_correctness_Lambda}]
The primal problem \eqref{eq:primal} for updating $\Lambda$ has convex objective function and is strictly feasible under the hypothesis that $\mathcal{D}_{n-1}$ consists of nonzero tensor slices $\D_{i}$. Hence the primal problem \eqref{eq:primal} obtains strong duality (see, e.g., \cite{boyd2004convex}). 

Let $\Lambda_n$ and $G_n$ be the optimizers of the primal problem \eqref{eq:primal} and the dual problem \eqref{eq:dual}, respectively. In what follows, we will apply Lemma~\ref{lem:saddle}. For $K = \Sigma^{N}_{I_{1}\times \cdots \times I_{d}}$, let us consider 
\begin{align*}
&\mathcal{J}: \R^{r\times N} \times \R^{I_{1}\times \cdots \times I_{d} \times N} \rightarrow (-\infty, + \infty]:\\ &(\Lambda,Y) \mapsto
\begin{cases}
\sum_{i=1}^{N} \left\{ H_{\X_i}  \left( \mathcal{D}_{n-1}\times_{d+1} \Lambda[:,i] - Y[:,i] \right) + \tau_{n} F_{\Lambda_{n-1}[:,i]}(\Lambda[:,i])\right\}
&\hbox{ if } \mathcal{D}_{n-1}\times_{d+1} \Lambda \in Y+K,\\
+\infty, &\hbox{ if } \mathcal{D}_{n-1}\times_{d+1} \Lambda \notin Y+K.
\end{cases}
\end{align*}
Then, \eqref{eq:primal} and \eqref{eq:dual} are the primal problem and the dual problem of $\mathcal{J}$, respectively.

From Cor. 19.17 in \cite{bauschke2011convex}, $(\Lambda_{n}, -G_{n})$ is a saddle point of the Lagrangian associated with $\mathcal{J}$. Applying Lemma~\ref{lem:saddle}(3), we get $\mathcal{D}_{n-1}\times_{\leq d} G_n \in \tau_n \partial F(\Lambda_n)$. Note that $\xi + J^\circ \otimes c^{\circ}_n \in \partial F(\Lambda_n)$ for any $c^\circ_n \in \mathbb{R}^{N \times 1}$, where all entries of $J^\circ \in \R^{r \times 1}$ are one, and $\xi \in \mathbb{R}^{r\times N}$ satisfies
\begin{align}\label{eq:Lambda_subdifferential2}
    \begin{cases}
        \xi[i,j] = \Lambda_n[i,j] - \Lambda_{n-1}[i,j] &\hbox{ if } \Lambda_n[i,j] > 0,\\
        \xi[i,j] \in (-\infty, - \Lambda_{n-1}[i,j]] &\hbox{ if } \Lambda_n[i,j] = 0,
    \end{cases}
\end{align}
for all $i = 1,2,\cdots, r$, and $j= 1,2,\cdots, N$. 
Hence 
\begin{align}\label{eq:Lambda_subdifferential1}
   \mathcal{D}_{n-1}\times_{\leq d} G_n/\tau_n = \xi + J^\circ \otimes c^{\circ}_n
\end{align}
for some $\xi$ satisfying \eqref{eq:Lambda_subdifferential2} and $c_{n}^{\circ}\in \mathbb{R}^{N \times 1}$. Now combining \eqref{eq:Lambda_subdifferential1} and  \eqref{eq:Lambda_subdifferential2} yields \eqref{eq:lam}. Finally, since we must have $\Lambda_n \in \textnormal{dom}(F)$, 
$c^\circ_n \in \mathbb{R}^{N \times 1}$ should be such that $\Lambda_{n}$ in \eqref{eq:lam} satisfies $\Lambda_n \in \Sigma^{N}_{r}$.
\end{proof}

\section{Proof of Theorem~\ref{lem:per_iter_correctness_WCPDL}}
\label{ap:WCPDL}

Here, we only prove the following proposition. The rest of arguments is parallel to the proof of Theorem~\ref{lem:per_iter_correctness}

\begin{prop}
\label{prop:per_iter_correctness_D_WCPDL}
For each $k \in \{1,2, \cdots, d\}$, let $\overline{\Lambda} \in \mathbb{R}^{I_1 \times I_2 \times \cdots \times I_{k-1} \times r \times  I_{k+1} \times \cdots \times I_{d} \times N}$ be obtained from 
\begin{align}
    \Out(U_{n}^{(1)},\dots,U_{n}^{(k-1)},U_{n-1}^{(k+1)}, \dots, U_{n-1}^{(d)}, \Lambda_n^T)  \in \mathbb{R}^{I_1 \times I_2 \times \cdots \times I_{k-1} \times I_{k+1} \times \cdots \times I_{d} \times N \times r}
\end{align}
by inserting the last mode into the $k$th mode. Let $U^{(k)}_{n} \in \Sigma^{r}_{I_{k}}$ be a solution of \eqref{alg:dWDL_high-level_D}. Suppose each fiber of $\Lambda_{n}\in \Sigma^{N}_{r}$ along the $k$th mode is not identically zero. Then 
$U^{(k)}_{n}$ is uniquely determined by 
\begin{align}
U^{(k)}_{n} = \left(U^{(k)}_{n-1} + \frac{G^{\dagger}_n \times_{\neq k} \overline{\Lambda}}{\tau_n}- J^{\dagger} \otimes c^{\dagger}_n \right)_+.
%\mathcal{D}_{n} = \left(\mathcal{D}_{n-1} + \frac{G^{\dagger}_n \times_{d+1} \Lambda^{T}_n}{\tau_n}-J^{\dagger} \otimes c^{\dagger}_n\right)_+.
\end{align}
Here, $G^{\dagger}_n$ is defined as the unique solution of the dual problem \eqref{eq:dual2}, $c^{\dagger}_n \in \mathbb{R}^{r\times 1}$ is chosen to satisfy $U^{(k)}_{n} \in \Sigma^{r}_{I_{k}}$ and all entries of $J^\dagger \in \R^{I_{k} \times 1}$ are one. 
\end{prop}

\begin{proof}
We first obtain the dual of \eqref{alg:WCPDL_high-level_D}:
\begin{align*}
&\min_{U \in \Sigma_{I_{k}}^r
} \,\, \left( \sum_{i=1}^{N} H_{\X_i}
(\overline{\Lambda}[:,i] \times_k U^T) \right) 
  + \tau_{n} F_{U_{n-1}^{(k)}}(U)\\
&\quad =\min_{U \in \Sigma_{I_{k}}^r, Q \in \Sigma^{N}_{I_{1}\times \cdots \times I_{d}}, Q[:,i] = \overline{\Lambda}[:,i] \times_k U^T
} \,\, \left( \sum_{i=1}^{N} H_{\X_i}
(Q[:,i]) \right) 
  + \tau_{n} F_{U_{n-1}^{(k)}}(U),\\  
&\quad= \min_{U \in \Sigma_{I_{k}}^r, Q \in \Sigma^{N}_{I_{1}\times \cdots \times I_{d}}} \max_{G \in \R^{I_{1}\times \cdots \times I_{d} \times N}} \,\, \sum_{i=1}^{N} \left\{ H_{\X_i}  \left( Q[:,i] \right) + \langle Q[:,i] - \overline{\Lambda}[:,i] \times_k U^T, G[:,i] \rangle\right\} + \tau_{n} F_{U_{n-1}^{(k)}}(U),\\
&\quad\overset{(a)}{=} \max_{G \in \R^{I_{1}\times \cdots \times I_{d} \times N}} \min_{U \in \Sigma_{I_{k}}^r, Q \in \Sigma^{N}_{I_{1}\times \cdots \times I_{d}}}  \,\, \sum_{i=1}^{N} \left\{ H_{\X_i}  \left( Q[:,i] \right) + \langle Q[:,i] - \overline{\Lambda}[:,i] \times_k U^T, G[:,i] \rangle\right\} + \tau_{n} F_{U_{n-1}^{(k)}}(U),\\
&\quad= \max_{G \in \R^{I_{1}\times \cdots \times I_{d} \times N}}  \,\, \sum_{i=1}^{N} - \left\{  \max_{Q \in \Sigma^{N}_{I_{1}\times \cdots \times I_{d}}} \langle Q[:,i] , -G[:,i] \rangle - H_{\X_i}  \left( Q[:,i] \right) \right\} \\
&\qquad \qquad + \min_{U \in \Sigma_{I_{k}}^r }  \tau_{n} F_{U_{n-1}^{(k)}}(U)  -\sum_{i=1}^{n} \langle  \overline{\Lambda}[:,i] \times_k U^T, G[:,i] \rangle ,\\
&\quad \overset{(b)}{=} - \min_{G \in \mathbb{R}^{I_{1}\times \cdots \times I_{d} \times N}} \left[ \sum_{i=1}^{N} \left\{ H_{\X_i}^*(-G[:,i])\right\} +  \max_{U \in \Sigma_{I_{k}}^r} \left\{  \langle \overline{\Lambda} \times_k U^T, G \rangle - \tau_n F_{U_{n-1}^{(k)}}(U) \right\} \right],\\
&\quad \overset{(c)}{=} - \min_{G \in \mathbb{R}^{I_{1}\times \cdots \times I_{d} \times N}} \sum_{i=1}^{N} \left\{ H_{\X_i}^*(-G[:,i])\right\} + \tau_n F_{U_{n-1}^{(k)}}^*(G \times_{\neq k} \overline{\Lambda}/\tau_n).
\end{align*}
Here, (a) uses strong duality for convex objectives; (b) uses the fact that 
\begin{align}
   (\overline{\Lambda} \times_k U^T) &= [ \overline{\Lambda}[:,1],\dots,\overline{\Lambda}[:,N] ]\times_k U^T  = \left[ \overline{\Lambda}[:,1]\times_{k} U^{T},\dots,\overline{\Lambda}[:,N]\times_{k} U^{T} \right],
\end{align} 
and (c) follows from the identity $\langle \overline{\Lambda} \times_k U^T, G \rangle = \langle U, G \times_{\neq k} \overline{\Lambda} \rangle$, which is easily verified from the definition. Then we can conclude similarly as in the proof of Proposition \ref{prop:per_iter_correctness_Lambda} by using Lemma \ref{lem:saddle} with $K = \Sigma^{N}_{I_{1}\times \cdots \times I_{d}}$ and
\begin{align*}
&\mathcal{J}: \R^{I_k \times r} \times \R^{I_{1}\times \cdots \times I_{d} \times N} \rightarrow (-\infty, + \infty]:\\ &(U,Y) \mapsto
\begin{cases}
\sum_{i=1}^{N}  H_{\X_i}  \left( \overline{\Lambda}[:,i] \times_k U^T - Y[:,i] \right)  + \tau_{n} F_{U_{n-1}^{(k)}}(U)
&\hbox{ if } \overline{\Lambda} \times_k U^T \in Y+K,\\
+\infty, &\hbox{ if } \overline{\Lambda} \times_k U^T \notin Y+K.
\end{cases}
\end{align*}
%\commHL{Revise the application of Lemma \ref{lem:saddle}} 
\end{proof}

\section{Auxiliary lemmas}
	
	\begin{lemma}
	\label{lem:fdual}
    Fix $g\in \R^{r}$ and let $\Sigma_{r}:=\{ (x_{1},\dots,x_{r})\in \R^{r}_{\ge 0}\,:\, \sum_{i=1}^{r}x_{i}=1 \}$. 
	The optimality condition of the problem
	\begin{align}
	\label{eq:fdual}
	    \sup_{\lambda \in \Sigma_{r}} \langle g, \lambda \rangle - \frac{1}{2}\lVert \lambda - \lambda_0 \rVert_{F}^{2}
	\end{align}
	is given as 
    \begin{align}
	    \lambda^* = (g + \lambda_0 - c 1_{r})_+
	\end{align}
	where $c$ is a constant chosen to satisfy $\lambda ^{*}\in \Sigma_{r}$.
	\end{lemma}
	
	\begin{proof}
	As the cost function of \eqref{eq:fdual} is strictly concave and $\Sigma_{r}$ is a closed set, there exists a unique maximizer $\lambda^* \in \Sigma_{r}$ of \eqref{eq:fdual}. For any $\epsilon \in [0,1]$ and $\lambda \in \Sigma_{r}$, consider $$h(\epsilon) := \langle g, \lambda^* + \epsilon(\lambda - \lambda^*) \rangle - \frac{1}{2}\lVert \lambda^* + \epsilon(\lambda - \lambda^*) - \lambda_0 \rVert_{F}^{2}.$$ 
	Noting that $\lambda^* + \epsilon(\lambda - \lambda^*)$ is also in $\Sigma_{r}$ for any $\epsilon \in [0,1]$.
	
	As $h(\epsilon)$ attains its maximum at $\epsilon = 0$, we have that for all $\lambda \in \Sigma_{r}$
	\begin{align}
	    0 \geq h'(0) = \langle g - \lambda^* + \lambda_0, \lambda - \lambda^* \rangle.
	\end{align}
	For $I_1 := \{ i \in \{1,2, \cdots, r\}: \lambda^*[i]>0\}$ and $I_2 := \{ i \in \{1,2, \cdots, r\}: \lambda^*[i]=0\}$, we obtain
	\begin{align}
	    0 \geq \sum_{i \in I_1}(g - \lambda^* + \lambda_0)[i] \times (\lambda - \lambda^*)[i] +  \sum_{i \in I_2}(g - \lambda^* + \lambda_0)[i] \times \lambda [i].
	\end{align}
	As $\lambda \in \Sigma_{r}$ is arbitrary, there exists a constant $c \in \mathbb{R}$ such that for $i \in I_1$
	\begin{align}
	    (g - \lambda^* + \lambda_0)[i] = c.
	\end{align}
	This yields that 
	\begin{align}
	     0 &\geq \sum_{i \in I_1} c \times (\lambda - \lambda^*)[i] +  \sum_{i \in I_2}(g - \lambda^* + \lambda_0)[i] \times \lambda [i],\\
	     &= \sum_{i \in I_2}(g - \lambda^* + \lambda_0[i] - c) \times \lambda [i].
	\end{align}
	The last equality is due to $\lambda, \lambda^* \in \Sigma_{r}$, 
	As a consequence, $(g - \lambda^* + \lambda_0)[i] \leq c$ and we conclude.
	\end{proof}
	
	\begin{lemma}\label{lem:positive_convergence_lemma}
		Let $(a_{n})_{n\ge 0}$ and $(b_{n})_{n \ge 0}$ be sequences of nonnegative real numbers such that $\sum_{n=0}^{\infty} a_{n}b_{n} <\infty$. Then 
		\begin{align}
			\min_{1\le k\le n} b_{k} \le \frac{\sum_{k=0}^{\infty} a_{k}b_{k}}{\sum_{k=1}^{n} a_{k}}  = O\left( \left( \sum_{k=1}^{n} a_{k} \right)^{-1} \right).
		\end{align}
	\end{lemma}
	
	\begin{proof}
		The assertion follows from noting that
		\begin{align}
			\left( \sum_{k=1}^{n}a_{k} \right) \min_{1\le k \le n} b_{k}\le \sum_{k=1}^{n} a_{k}b_{k} \le  \sum_{k=1}^{\infty} a_{k}b_{k} <\infty.
		\end{align}
		%The proof of \textbf{(ii)} is omitted but can be found in \cite[Lem. A.5]{mairal2013stochastic}.
	\end{proof}
	
	\begin{lemma}[Convex Surrogate for Functions with Lipschitz Gradient]
		\label{lem:surrogate_L_gradient}
		Let $f:\R^{p}\rightarrow \R$ be differentiable and $\nabla f$ be $L$-Lipschitz continuous. Then for each $\param,\param'\in \R^{p}$, 
		\begin{align}
			\left| f(\param') - f(\param) - \langle \nabla f(\param),\, \param'-\param\rangle  \right|\le \frac{L}{2} \lVert \param-\param'\rVert^{2}.
		\end{align}
	\end{lemma}
	
	\begin{proof}
		This is a classical Lemma (see, e.g., Lem 1.2.3 in \cite{nesterov1998introductory}). We include a proof of this statement for completeness. First write 
		\begin{align}
		    f(\param') - f(\param) = \int_{0}^{1} \left\langle \nabla f \left( \param + s(\param'-\param) \right),\, \param'-\param \right\rangle \,ds.
		\end{align}
		 By Cauchy-Schwarz inequality and $L$-Lipscthizness of $\nabla f$, 
		 \begin{align}
		   \left| \int_{0}^{1} \left\langle \nabla f \left( \param + s(\param'-\param) \right),\, \param'-\param \right\rangle - \int_{0}^{1} \left\langle \nabla f \left( \param \right),\, \param'-\param \right\rangle \,ds  \right| &\le \int_{0}^{1} \left\lVert \nabla f \left( \param + s(\param'-\param) \right) - \nabla f \left( \param \right) \right\rVert\, \lVert \param'-\param \rVert \,ds \\
		   &\le \int_{0}^{1} L s \lVert \param'-\param \rVert^{2}\,ds \\
		   &= \frac{L}{2}\lVert \param-\param' \rVert^{2}.
		 \end{align}
		 Then the assertion follows. 
	\end{proof}

	\begin{lemma}[Second-Order Growth Property]
		\label{lem:second_order_growh_univariate}
		Let $g:\R^{p} \rightarrow [0,\infty)$ be  $\mu$-strongly convex and let $\Param$ is a convex subset of $\R^{p}$. Let $\param^{*}$ denote the minimizer of $g$ over $\param$. Then for all $\param\in \param$,
		\begin{align}
			g(\param) \ge  g(\param^{*}) + \frac{\mu}{2} \lVert \param-\param^{*} \rVert^{2}.	
		\end{align}
	\end{lemma}
	
	\begin{proof}
		See Lem. B.5 in \cite{mairal2013optimization}.
	\end{proof}

	\begin{lemma}[Characterization of weak convexity]\label{lem:weak_convexity}
		Let $f:\R^{p}\rightarrow \R$ be a smooth function. Fix a convex set $\Param\subseteq \R^{p}$ and $\rho>0$. The following conditions are equivalent. 
		\begin{description}[itemsep=0.1cm]
			\item[(i)] (Weak convexity) $\param\mapsto f(\param) + \frac{\rho}{2}\lVert \param \rVert^{2}$ is convex on $\Param$;
			\item[(ii)] (Hypermonotonicity) $ \langle \nabla f(\param) - \nabla f(\param'),\, \param-\param'  \rangle \ge - \rho\lVert \param-\param' \rVert^{2}$ for all $\param,\param'\in \Param$; 
			
			\item[(iii)] (Quadratic lower bound) $f(\param) - f(\param') \ge \langle \nabla f(\param'),\, \param-\param' \rangle - \frac{\rho}{2}\lVert \param-\param' \rVert^{2}$ for all $\param,\param'\in \Param$. 
		\end{description}
	\end{lemma}
	
	\begin{proof}
    See Lem. B.2 in \cite{lyu2022convergence}. See also Thm. 7 in \cite{daniilidis2005filling} for an equivalent statement for a more general case of locally Lipschitz functions. 
	\end{proof}

	\begin{lemma}\label{lem:L_smooth_weak_convex}
	Let $f:\R^{p}\rightarrow \R$ be a function such that $\nabla f$ is $L$-Lipscthiz for some $L>0$. Then $f$ is $L$-weakly convex, that is, $\param\mapsto f(\param)+\frac{L}{2}\lVert \param\rVert^{2}$  is convex. 
	\end{lemma} 
	
	\begin{proof}
	    Follows immediately by Lemmas \ref{lem:surrogate_L_gradient} and  \ref{lem:weak_convexity}.
	\end{proof}

%%%%%%%%%%%%%%%%%%%%%%%%%%%%%%%%%%%%%%%%%%%%%%%%%%%%%%%%%%%%%%%%%%%%%%%%%%%%%%%
%%%%%%%%%%%%%%%%%%%%%%%%%%%%%%%%%%%%%%%%%%%%%%%%%%%%%%%%%%%%%%%%%%%%%%%%%%%%%%%

\end{document}